%% file: FlowTheoryManuscript.tex
 \documentclass[preprint, 12pt]{elsarticle}
 \pdfoutput=1

\usepackage{times}
\usepackage{epsfig}
\usepackage{graphicx}
\usepackage{amsmath}
\usepackage{amsthm}
\usepackage{amssymb}
\usepackage{mathrsfs}
\newtheorem{theorem}{Theorem}[section]
\newtheorem{proposition}[theorem]{Proposition}
\newtheorem{definition}[theorem]{Definition}
\newtheorem{appxlem}{Lemma}

\newcommand{\rf}{{\rm ref}}
\newcommand{\bfx}{{\bf x}}

\newcommand{\bft}{{\bf t}}

\newcommand{\reals}{{\mathbb R}}

\biboptions{}

\journal{ACHA}

\begin{document}
\begin{frontmatter}

\title{{\bf A Theory for Optical flow-based Transport \\ on Image Manifolds}}
\author{Sriram Nagaraj, Aswin C. Sankaranarayanan, Richard G. Baraniuk}
\address{Rice University}

\begin{abstract}
\input{abstract.tex}
\end{abstract}


\end{frontmatter}

\section{Introduction}\label{sec:intro}
\input{intro1.tex}

\section{OFMs and the Flow Metric}\label{sec:theory}
\input{theory.tex}

\section{Geometric Tools for IAMs via the Flow Metric}\label{sec:difftools}
\input{difftools.tex}

\section{Multiscale Structure of Parallel Flow Fields}\label{sec:multiscale}
\input{multiscale.tex}

\section{Discussion}\label{sec:discuss}
\input{discuss.tex}

\section*{Acknowledgments}
The authors thank Chinmay Hegde for valuable discussions and insightful comments on the manuscript.

This work was partially supported by the grants NSF CCF-0431150, CCF-0728867, CCF-0926127, CCF-1117939, ARO MURI W911NF-09-1-0383, W911NF-07-1-0185, DARPA N66001-11-1-4090, N66001-11-C-4092, N66001-08-1-2065, AFOSR FA9550-09-1-0432, and LLNL B593154.

\appendix

\section{Smoothness and Distance Properties of OFMs} \label{sec:smooth}
\input{smoothproofs.tex}

\bibliographystyle{elsarticle-num}
\bibliography{opflowTheory}

\end{document}

%% file: abstract.tex
An image articulation manifold (IAM) is the collection of images formed when an object is articulated in front of a camera.  
IAMs arise in a variety of image processing and computer vision applications, where they provide a natural low-dimensional embedding of the collection of high-dimensional images. 
To date IAMs have been studied as embedded submanifolds of Euclidean spaces. 
Unfortunately, their promise has not been realized in practice, because real world imagery typically contains sharp edges that render an IAM non-differentiable and hence non-isometric to the low-dimensional parameter space under the Euclidean metric.
As a result, the standard tools from differential geometry, in particular using linear tangent spaces to transport along the IAM, have limited utility.  
In this paper, we explore a nonlinear transport operator for IAMs based on the {\em optical flow} between images and  develop new analytical tools reminiscent of those from differential geometry using the idea of \emph{optical flow manifolds} (OFMs). 
We define a new metric for IAMs that satisfies certain local isometry conditions, and we show how to use this metric to develop a new tools such as flow fields on IAMs, parallel flow fields, parallel transport, as well as a intuitive notion of curvature. 
The space of optical flow fields along a path of constant curvature has a natural multi-scale structure via a monoid structure on the space of all flow fields along a path. 
We also develop lower bounds on approximation errors while approximating non-parallel flow fields by parallel flow fields. 

%% file: intro1.tex
\subsection{Image articulation manifolds}
Many problems in image processing and computer vision involve image ensembles that are generated by varying a small set of imaging parameters such as pose, lighting, view angle, etc.\ of a fixed
three-dimensional (3D) scene.  As the parameters vary, the images can be modeled as a lying on a (typically nonlinear) manifold called an \emph{image articulation manifold (IAM)} \cite{DonohoGrimes, Wakin, coifman1, culpepper1}.  Each point on an IAM is  an image at a particular parameter value.
Over the past decade, there has been significant work \cite{coifman1,ISOMAP, LLE, LLLE} in learning and processing the underlying geometric structures associated with image ensembles.
For instance, tasks such as recognition, classification, and image
synthesis can be interpreted as navigation along a particular IAM. 

More specifically, we define an image articulation manifold (IAM) as the set of images $M$ formed by the action of an imaging map $i$ on a space of articulations $\Theta$, i.e, $M = \{i_\theta = i(\theta): \: \theta \in \Theta \}$.
This imaging process can be decomposed into two steps: first the action of the
articulation on a 3D object or scene and then the subsequent imaging of the
articulated object/scene (see \cite{3DBook} for a detailed discussion of image
formation).\footnote{Before we proceed further, it is worth discussing certain
degeneracies in the imaging process that we wish to avoid in this paper for 
analytical reasons. In particular, we want to avoid cases
where the set $M$ is not expressive of the full range of articulations. 
As an example, consider a uniformly colored 3D sphere $\mathcal{O}$ undergoing rotation about a fixed axis. Here, the parameter space $\Theta$ is the unit circle $S^{1}$. 
Being uniformly colored, the sphere does not change appearance under rotation, and the 
IAM degenerates to a single point $\{I_{\mathcal{O}}\}$. 
Were the sphere richly textured, we would obtain new views of the sphere for each rotation
so that the IAM is isomorphic to $S^{1}$. 
For the remainder of the paper, we will consider IAMs without degeneracies by assuming that the imaging map $i$ is a re-parametrization of $\Theta$; i.e., we will assume that the IAM is homeomorphic to the corresponding parameter space.
}

\subsection{IAM Challenges} \label{sec:short}
In spite of much progress, there are fundamental challenges to successfully applying manifold-processing tools to generic image data, in particular IAMs.

First, it has been shown that  IAMs containing images with sharp edges are \emph{non-differentiable} \cite{DonohoGrimes}.
Specifically, given parameters $\theta_{1}$ and $\theta_{2}$ with corresponding images $I_{1}$ and $I_{2}$, it follows that the $L^{2}$ distance $\|I_{1}-I_{2}\|_{L^{2}}$ between images $I_{1}$ and $I_{2}$ is a nonlinear function $\eta(\|\theta_{1}-\theta_{2}\|)$ that is asymptotically equivalent to $({\|\theta_{1}-\theta_{2}\|})^{\frac{1}{2}}$. Indeed, $\frac{\|I_{1}-I_{2}\|_{L^{2}}}{\|\theta_{1}-\theta_{2}\|}\geq c\|\theta_{1}-\theta_{2}\|^{-\frac{1}{2}}$, and this non-Lipschitz relation indicates that the corresponding IAM is non-differentiable.
Non-differentiability suggests that local linear approximations, such as those suggested
by differential geometry, are invariably inaccurate on IAMs. 
To illustrate this, consider a stylized example of image interpolation (see Fig.\ \ref{fig:OFM_Transport_Linear1}). 
Given images $I_{1},I_{2}\in M$, consider the affine path $\alpha(t)=tI_{1}+(1-t)I_{2}$ with $t\in [0,1]$; for a smooth 
manifold $( > {\cal{C}}^2)$, this line would be a close approximation to the actual manifold, especially over small neighborhoods.
However, non-differentiability implies that a first-order approximation to the manifold is inaccurate even over a small neighborhood. This is illustrated in Fig.\ \ref{fig:OFM_Transport_Linear1}.
To alleviate the non-differentiability problem, Wakin et al.~\cite{Wakin} have proposed a multiscale smoothing procedure that regularizes each point of the IAM by 
a set of multiscale Gaussian smoothing filters that render the IAM smooth across the various scales. This smoothing procedure then enables the definition of linear tangent spaces on which one can perform standard linear methods of analysis.
However, this is unsatisfactory, since Gaussian smoothing is inherently lossy --- leading to loss of high-frequency information in the images. 
Further,  defining tangent vectors as the limit of a multiscale procedure is inherently complex and moreover not possible
for practical scenarios where we have only samples from the IAM.

\begin{figure}[t]
\centering
\includegraphics[width=0.8\textwidth]{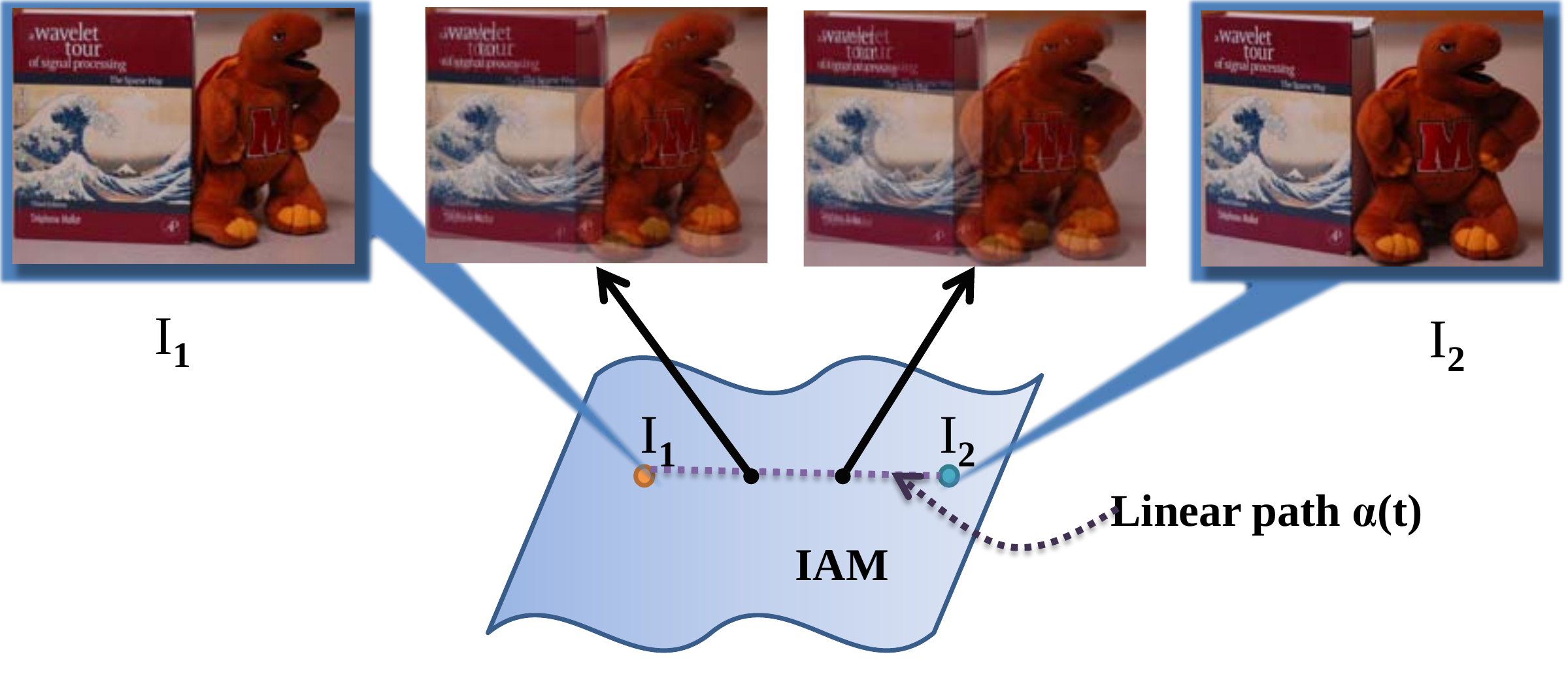}
\caption{Non-differentiability of IAMs renders  (locally) linear transport inaccurate. Consider an
image interpolation task. Given images $I_{1}$ and $I_{2}$, the line connecting them $\alpha(t) = tI_1+(1-t)I_2$ is blurred and thus a poor approximation
to the geodesic connecting $I_1$ and $I_2$.}
\label{fig:OFM_Transport_Linear1}
\end{figure}

\sloppy
Second, conventional manifold models lack a meaningful metric between points on the IAM, especially when the sampling of the manifold is sparse.
Consider the simple {\em translation manifold} $M_{T}$ generated by imaging a black disk of radius $R$ translating on an infinite white background (see Fig.\ \ref{fig:TranslatingDisk}). Let $I_{1}$ be the image of the disk with center $c_{1}$ and $I_{2}$ the image of the disk with center $c_{2}$. It then follows that 
\begin{equation*}
\|I_{1}-I_{2}\|_{L^{2}}= \eta(\min(2R, \| c_1-c_2\|)) = \begin{cases}\eta(\|c_{1}-c_{2}\|), & \|c_{1}-c_{2}\| < 2R \\ \eta(2R), & \|c_{1}-c_{2}\| \geq 2R\end{cases}
\end{equation*} 
where $\eta(\cdot)$ is a nonlinear function. Thus, when the disk centers are separated by a distance greater than the diameter of the disk, the metric is completely uninformative (see Fig.\ \ref{fig:TranslatingDisk}).
This suggests that, unless the sampling of images from the IAM is sufficiently dense, organization of the images using a construct such as a $k$-nearest neighbor graph is meaningless. $k$-nearest neighbor graphs are at the heart of traditional manifold learning techniques such as LLE \cite{LLE}, ISOMAP \cite{ISOMAP}, and diffusion maps \cite{coifman1}; as a consequence, such techniques are doomed to fail unless the sampling is sufficiently dense.

\begin{figure}[t]
\centering
\includegraphics[width=\textwidth]{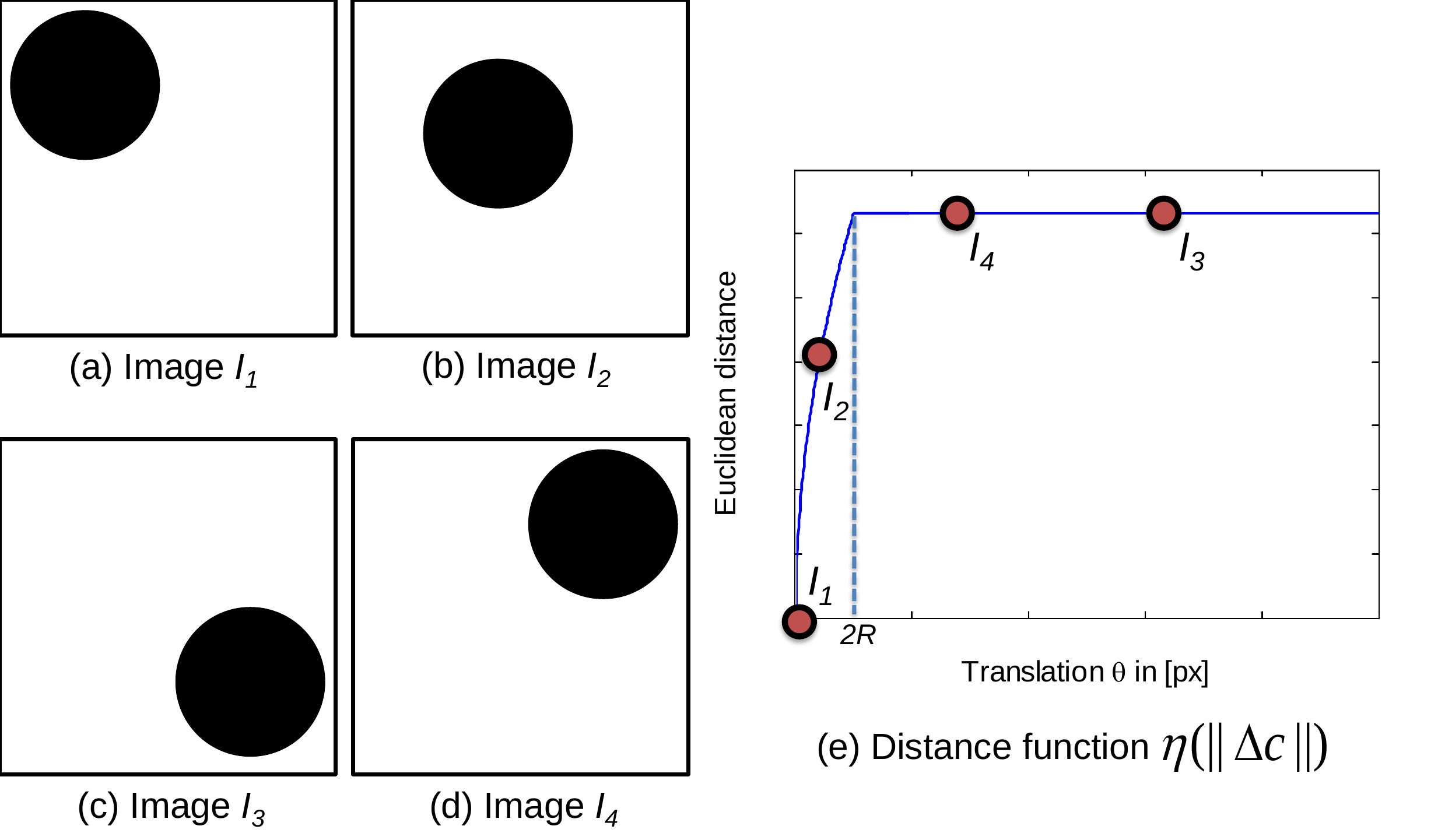}
\caption{Euclidean distance between images on an IAM can be meaningless. (a-d) Images of a translating disk over a white background. (e) The $L^2$ distance between two images depends on the amount of overlap between the the disks. However, the overlap is zero when the distance between centers exceed $2R$; in the figure above, $d(I_1, I_3) = d(I_1, I_4)$. 
Hence, the distance between two images can be written as $\eta( \min(2R, \| c_1 - c_2 \|_2 ))$ where $\eta(\cdot)$ is a monotonically increasing function. Conventional manifold learning
techniques, which rely heavily on meaningful local distances, do not work unless the sampling on the manifold is dense.}
\label{fig:TranslatingDisk}
\end{figure}


These shortcomings are exacerbated for images with rich textures;
this 
rules out a consistent analysis of IAMs based on classical differential geometry that relies on the smoothness and metric properties of the manifold.
A critical missing link in the analysis of IAMs is a systematic theory that successfully handles the above issues and permits new mathematical tools for IAMs, including analogues of traditional notions such as curvature, vector fields, parallel transport, etc.\ \cite{docarmo}. 
Clearly, once such analytic tools are available, we can greatly expand the scope of applications and pave the way for efficient algorithms specific to IAMs.

\subsection{Transport operators}
Recently, a new class of methods for handling image ensembles has been developed based on the idea of transport operators \cite{culpepper1, rao1, grimson1, tuzel1, CVPROFM}.
A transport operator on an IAM is a (typically nonlinear) map from the manifold into itself that enables one to move between different points on the manifold. 
Given images $I_{1}(\textbf{x})$ and $I_{2}(\textbf{x})$, (where $I(\textbf{x})$ denotes the intensity at the spatial location $\textbf{x}=(x,y)\in [0,1]\times[0,1]$) a {\em transport operator} $T$ is a mapping that acts in the following fashion: 
\begin{equation}
I_{2}(\textbf{x})=I_{1}(\textbf{x}+T(\textbf{x})).
\label{eqn:transport}
\end{equation}
Instead of relying on the linear tangent space that accounts only for infinitesimal transformations, nonlinear transport operators  can be well-defined over larger regions on the IAM.


For certain classes of articulations, the associated transport operators have algebraic structure in the form of a Lie group \cite{grimson1,rao1,tuzel1,culpepper1}. 
In such instances, it is possible to explicitly compute transport operators that capture the curved geometric structure of an IAM. An example of this is the case of {\em affine} articulations, where the transport takes the form
\begin{equation*}
I({\bfx}) = I_0 (A\bfx+\bft).
\end{equation*}
In this case, the transport operator $T(\cdot)$ is of the form $T(\bfx) = (A - {\mathbb I})\bfx+\bft$; this can be modeled as the group of 2D affine transformations.
The affine group has found extensive use in tracking \cite{tuzel1} and registration \cite{li2010aligning}.
Miao and Rao~\cite{rao1} learn affine transport operators for image ensembles using a matrix exponential-based generative model and demonstrate improved performance over locally linear approximations.
Culpepper and Olshausen~\cite{culpepper1} extend this framework using a more complex model on the transport operator in order to model paths on image manifolds. 

Other common examples of articulations in computer vision are  the projective group (used to model homographies and projective transformations) and diffeomorphisms (used to model 1D warping functions, density functions)~\cite{miller1}. However, while algebraic transport methods are mathematically elegant, they are applicable only to a very restrictive class of IAMs. Many IAMs of interest in computer vision and image processing applications, including IAMs corresponding to 3D pose articulations and non-rigid deformations, possess no explicit algebraic structure.

\subsection{Optical flow-based transport}
In this paper, we study a specific class of transport operators that are generated by the {\em optical flow} between images (we introduced this notion empirically in \cite{CVPROFM}). 
Given two images $I_{1}$ and $I_{2}$, the optical flow between them is defined to be the tuple $(\textbf{v}_{x},\textbf{v}_{y})\in L^{2}([0,1]^{2})\times L^{2}([0,1]^{2})$ such that 
\begin{equation}
I_{2}(x,y)=I_{1}(x+\textbf{v}_{x}(x,y),y+\textbf{v}_{y}(x,y)).
\label{eq:oflow}
\end{equation}
A common assumption in computing the optical flow between images is {\em brightness constancy} \cite{HornSchunk}, where the spatial intensity is assumed to not change between $I_{1}$ and $I_{2}$. 
Since the pioneering work of Horn and Schunk \cite{HornSchunk}, there has been significant progress towards the robust estimation of optical flow between image pairs \cite{OpticalFlow2,OpticalFlow3}. 

Optical flow is a natural and powerful {transport operator} to transform one image into another. 
In the context of image manifolds, the collection of {all} optical flow operators at a point on an IAM is a manifold of the same dimension as the IAM \cite{CVPROFM}. In other words, at any reference image on an IAM, there is a corresponding manifold of optical flow operators that  transports the reference image along the IAM. 
This new operator manifold, which we christen the {\em optical flow manifold} (OFM), can be used to obtain a canonical chart for the IAM \cite{CVPROFM}; this  enables significantly improved navigation capabilities on the IAM than previous methods.
In particular, OFM-based transport is well-defined even in instances when the transport operators cannot be modeled as a Lie group; an example of this is the pose manifold (the IAM associated with rigid body motion).

To see the efficacy of optical flow-based transport, consider again the case of image interpolation,
 but now with optical flow as the transport operator (see Fig.\ \ref{fig:OFM_Transport_OF}). Here, the path $\gamma(t), t\in [0,1]$ generated on the IAM via optical flow is a better representative of a path on the IAM, i.e., $\gamma(t)\in M$ for all $t$. Moreover, if the IAM is generated via Lie group actions, then this path coincides with the geodesic.

\begin{figure}[t]
\centering
\includegraphics[width=\textwidth]{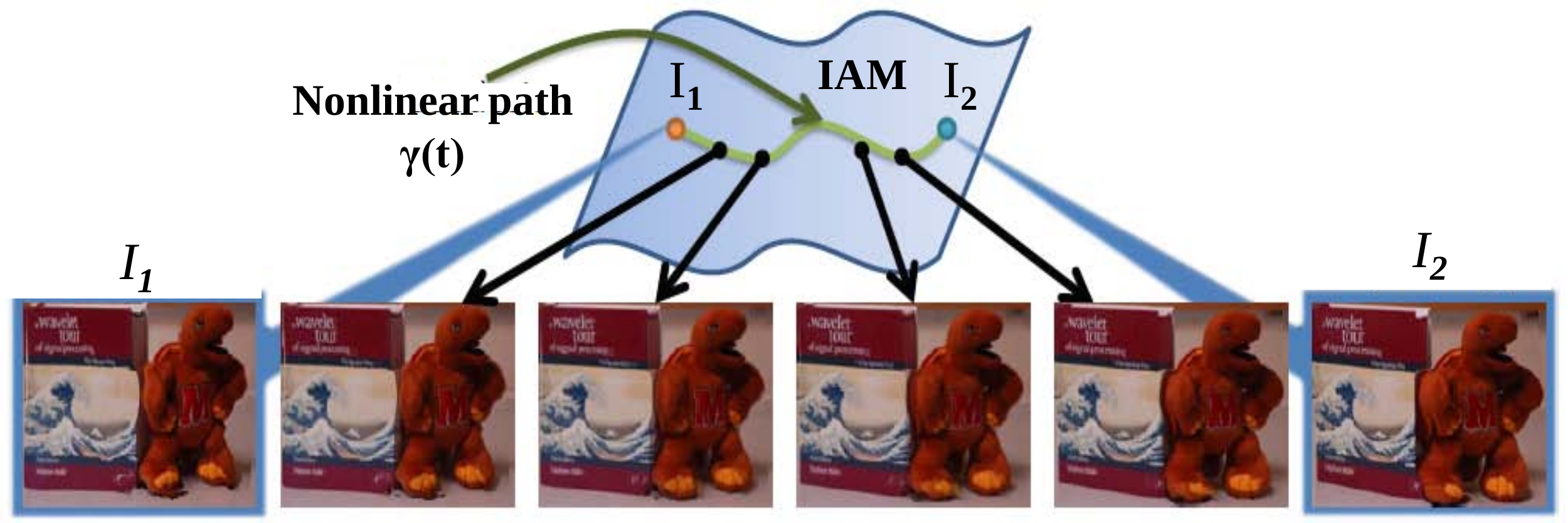}
\caption{Optical flow based transport on an IAM leads to accurate image interpolation. Consider again the interpolation task in Fig.\ \ref{fig:OFM_Transport_Linear1}. While locally linear transport on IAMs lead to inaccurate interpolation results,  the path $\gamma(t)$ generated using optical flow provides an accurate approximation to the true path between $I_{1}$ and $I_{2}$. Moreover, if the IAM is generated via Lie group actions, then the path generated coincides with the geodesic.}
\label{fig:OFM_Transport_OF}
\end{figure}


As discussed earlier, IAMs composed of images with sharp edges and textures lack smoothness and hence do not support locally linear modeling.
In contrast, for a large class of interesting articulations, including
affine transformations and 3D pose, the corresponding OFMs are smooth (see \ref{sec:smooth} and \cite{CVPROFM}) and support local linear modeling.  
Moreover, we can define the distance between two optical flows to measure the amount of motion required to articulate from one image to another. 
These properties suggest that machine learning techniques (such as LLE, ISOMAP, etc.) should be able to extract a considerable amount of geometrical information about an image ensemble when applied to its OFMs.
Figure \ref{fig:OFMvsIAMISOMAP1} confirms this fact by showcasing the improvement in OFM vs.\ IAM-based dimensionality reduction for an object under 2D translation.

The stability of OFM-based modeling enables us below to develop new differential geometric tools for signal processing such as the geometric mean and Karcher mean \cite{karcher1977riemannian} of an image ensemble. The Karcher mean of a set of images is defined as the image on the IAM that is closest to the set in the sense of \emph{geodesic distance}. Figure \ref{fig:piggy} shows the results of Karcher mean estimation using the OFM. 
The interested reader is referred to \cite{CVPROFM} for details of the Karcher mean estimation as well as more OFM-based tools and applications.

\subsection{Related work}
OFMs fall under the category of deformation modeling, which has been studied in many different contexts, including active shape \cite{cootes1995active} and active appearance models \cite{cootes2001active}.
However, there are significant differences between OFMs and traditional deformation modeling approaches.
Morphlets \cite{kaplan2006morphlet}, for examlpe, provide a mutliscale modeling of and interpolation across image deformations, but their treatment is limited to image pairs. In contrast, OFMs apply to image ensembles consisting of a potentially large number of images. 
Beymer and Poggio \cite{beymer1996image} have argued for the use of motion-based representations for learning problems. However, their goal is  image synthesis, and therefore they offer no insights into
the geometric nature of manifold-valued data.
Jojic et al.\ \cite{jojic2001learning} use a layered representation to represent videos by separating the appearance of moving objects from their motion and then representing each using subspace and manifold models, respectively. This simple, yet powerful, representation can model and synthesize complex scenes using simple primitives, but it is not intended to go beyond simple manifolds such as those generated by translations and affine transformations.

\begin{figure}[ttt]
\centering
\includegraphics[width=\textwidth]{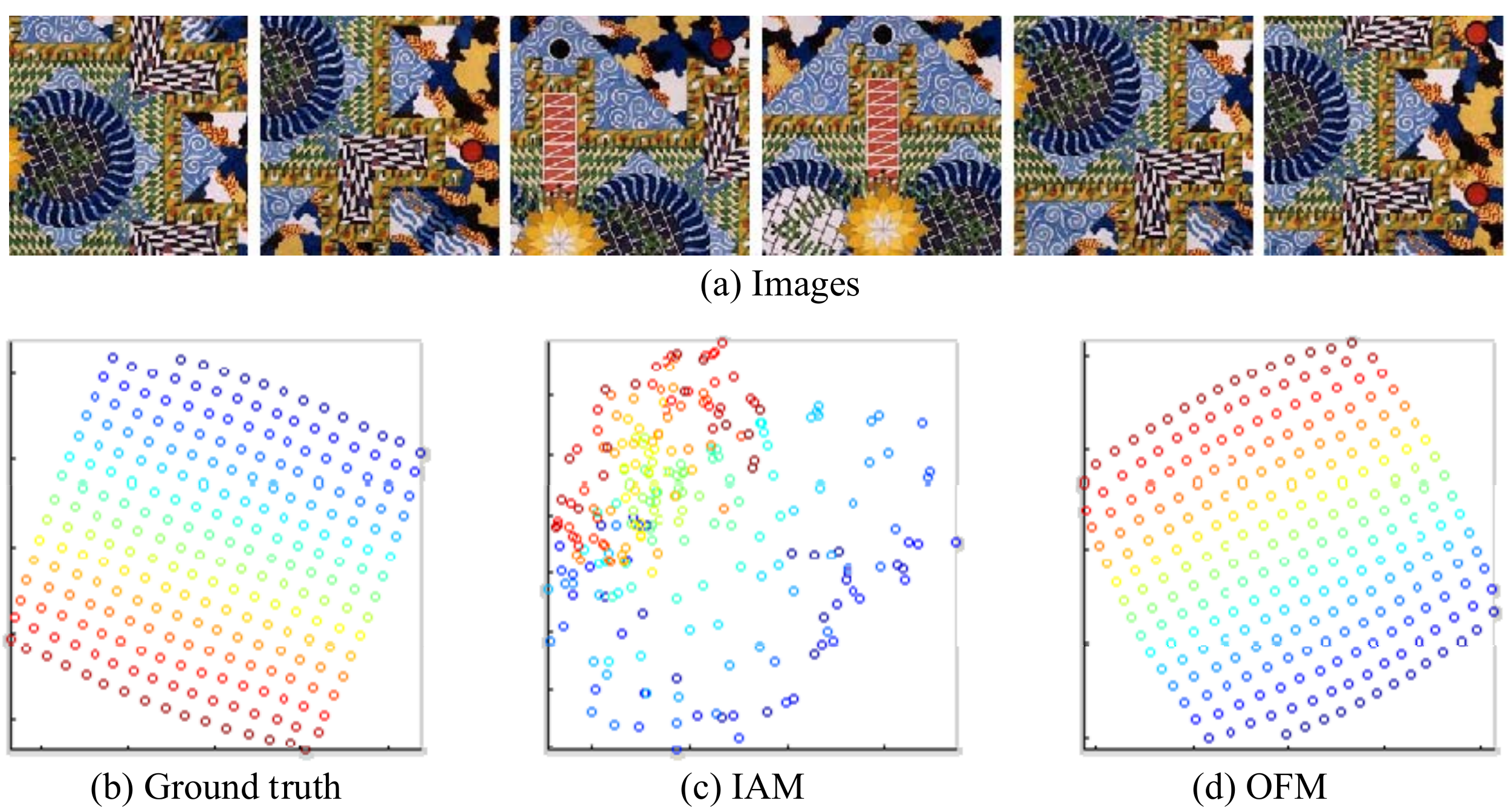}
\caption{Comparison of manifold learning on a IAM vs.\ OFM (from \cite{CVPROFM}). 
The IAM is generated by cropping patches of size $200 \times 200$ pixels at random from an image --- thereby generating a 2D translation manifold.
(a) Sample images from the IAM showing a few images at various translations. 
(b-d)  2D embedding obtained using ISOMAP on (b) the ground truth, (c) the  IAM, and (d) the OFM.
 The near perfect embedding obtained using OFM hints at the near-perfect isometry in the  OFM distances.}
\label{fig:OFMvsIAMISOMAP1}
\end{figure}

\begin{figure}[ttt]
\centering
\includegraphics[width=0.80\textwidth]{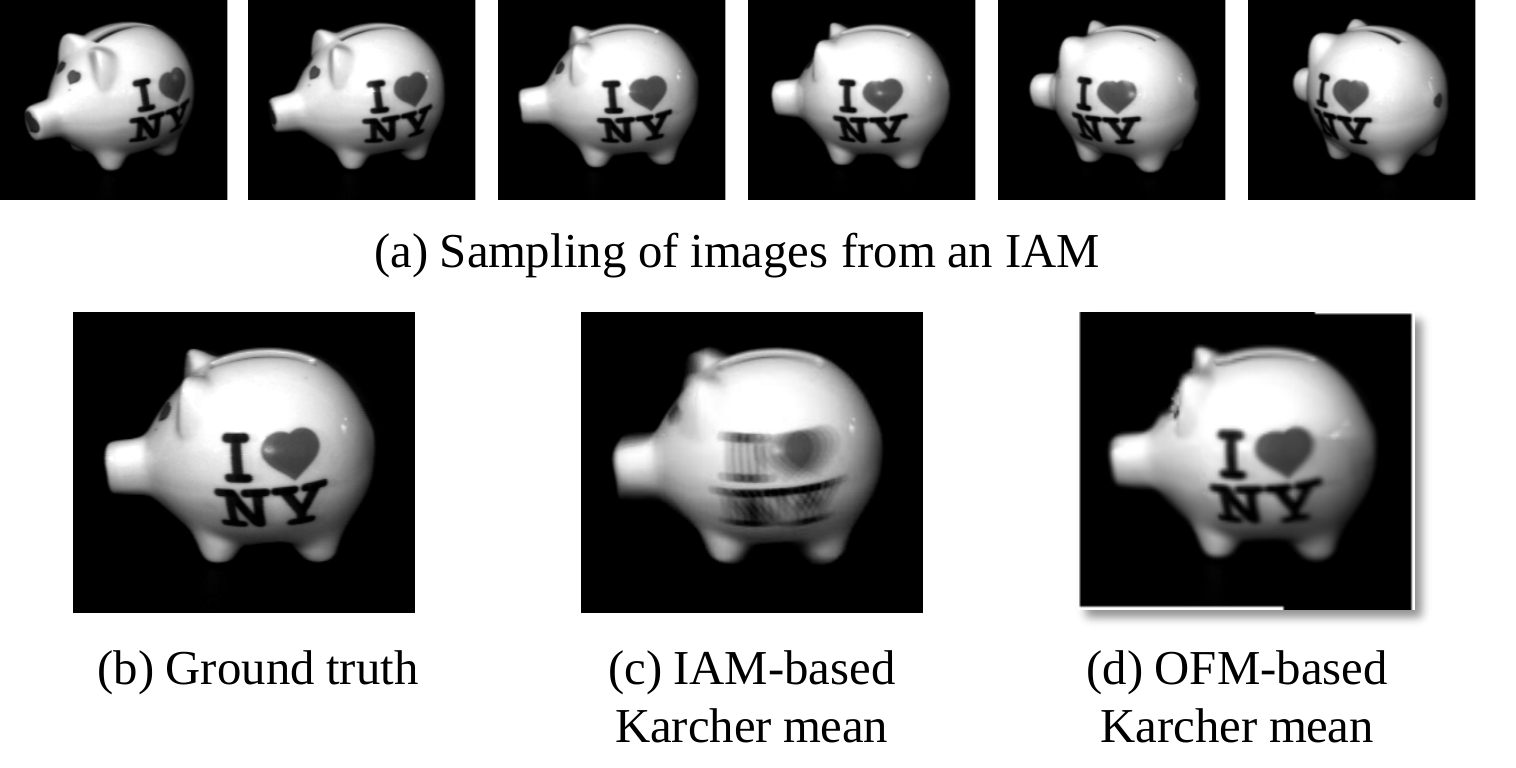}
\caption{Karcher mean estimation for 20 images 
generated by the rotation of an object about a fixed axis \cite{CVPROFM}.
The images are from the COIL dataset \cite{nene1996columbia}.
(a) Sample images from the IAM showing a few images at various rotations. 
We estimate the Karcher mean using local linear transport on the IAM and the OFM.
Shown above are  (b) the  ground truth Karcher mean of the images, (c) the Karcher mean estimated using local linear transport on the IAM, and
(d) the Karcher mean estimated using local linear transport on the OFM.
The accuracy of the estimate obtained from the OFM showcases the validity of local linear transport on the OFM.
}
\label{fig:piggy}
\end{figure}


\subsection{Specific contributions}
The main aim of this paper is to develop the mathematical foundation of optical flow-based transport operators for image manifolds, which were introduced empirically \cite{CVPROFM}.

We first define a metric on an IAM using its corresponding OFMs.
Each OFM has a natural metric that  is a locally isometric function of the corresponding parameter values. 
We consider the induced metric on the IAM, which we dub the \emph{flow metric} and show that the flow metric between two points of the IAM is a measure of the distance between the corresponding parameter values.
Next, using the flow metric, we develop analytic notions of curvature, optical flow fields, and parallel transport. We analyze in detail the case of optical flow fields defined along a fixed curve. In particular, we define a unique function associated to each such flow field, which we dub the \emph{motion function}, using which we can define the notion of \emph{parallel} flow fields. We answer the natural question of how one can optimally approximate a non-parallel field by a parallel field, and thereby induce uniform motion along the curve. We also construct a monoid structure on the set of all flow fields along a fixed curve. Under certain conditions on the curvature of a curve, we show that the space of parallel optical flow fields along the curve is a submonoid that comes with a convenient multi-scale structure.

We envision that the theory developed in this paper will enable a large class of practical applications involving image manifolds especially under sparse sampling of images from the manifold. In this context, we believe that the long term impact of this paper is the first step towards a complete theory of manifolds for arbitrary classes of signals using transport operators along the lines of the classic differential geometry for smooth manifolds.

\subsection{Organization}
The remainder of the paper is organized as follows. In Section \ref{sec:theory}, we introduce OFMs and, using a fixed metric on the OFMs, study the induced flow metric on the corresponding IAM. We compare dimensionality reduction techniques on the IAM vs.\ OFMs. Using the flow metric, we develop geometric tools on the IAM in Section \ref{sec:difftools} and highlight their application to parameter estimation. We also develop error bounds on approximating non-parallel flow fields by parallel flow fields and illustrate the idea with the example of video resampling. In Section \ref{sec:multiscale} we develop the multiscale structure of parallel flow fields. We conclude in Section \ref{sec:discuss} with a brief discussion.

%% file: theory.tex
In this section, we define and study the basic properties of OFMs corresponding to an IAM $M$.
Much of the section is concerned with a formal introduction to OFMs, first defined in \cite{CVPROFM}, leading towards the development of our fundamental tool, the flow metric on an IAM.

\subsection{Optical flow manifolds}
The optical flow between two images on an IAM measures the apparent motion between the two images and thus reflects the corresponding parameter change between the two images. For a fixed base image $m\in M$, consider a neighborhood $N(m)$ of $m$. If for $m'\in N(m)$ there exist flow vectors $(\textbf{v}_{x},\textbf{v}_{y})$ such that $m'$ can be obtained from $m$ using the flow vector, i.e., $m'(x,y)=m(x+\textbf{v}_{x},y+\textbf{v}_{y})$, then we say that optical flow exists from $m$ to $m'$. We denote this situation as $\phi_{\textbf{v}_{x},\textbf{v}_{y}}(m)=m'$ or simply $\phi(m)=m'$. In practice, occlusion or boundary effects (i.e., veiling of certain portions of an image due to changes in the background or interferers in the scene) may lead to undefined estimates for the flow vectors. However, as described in \cite{CVPROFM}, one can mitigate these issues by incorporating additional consistency tests to ensure that only the meaningful flow vectors are retained.


The set of all points $m'\in M$ for which optical flow from $m\in M$ to $m'$ exists is a neighborhood of $m$, which we denote by \begin{equation*}B_{m}=\{m'\in M: m'(x,y)=m(x+\textbf{v}_{x},y+\textbf{v}_{y})\}.\end{equation*}Using this neighborhood, we define the optical flow manifold $O_{m}$ at $m$ as follows.

\begin{definition} Let $M$ be a $K$-dimensional IAM. Given $m \in M$, the \emph{Optical Flow Manifold} $O_{m}$ is defined as the set of optical flows $\phi$ between $m$ and points in $B_{m}$ \begin{equation*}O_{m} = \{\phi=(\textbf{v}_{x},\textbf{v}_{y}):\phi(m)\in B_{m}\}.\end{equation*}
The neighborhood $B_{m}$ is called the \emph{flow neighborhood} around $m$. \end{definition}

$O_{m}$ is nonlinear, i.e., an arbitrary OFM is not always a linear vector space. It is clear that the collection of neighborhoods $\{B_{m}:m\in M\}$ covers $M$.

\begin{proposition}
Let $M$ be a $K$-dimensional IAM. Then $O_{m}$ has the structure of a manifold homeomorphic to $B_{m}.$
\end{proposition}
\begin{proof}
We first show that $O_{m}$ is homeomorphic to $B_{m}$ and pullback the smooth structure of $B_{m}$ to induce a smooth structure on $O_{m}$. Consider the map $g: B_{m}\rightarrow O_{m}$ that sends a point $m'\in B_{m}$ to $\phi \in O_{m}$ such that $\phi(m')=m$. $g$ is clearly bijective and hence, we use $g$ to endow $O_{m}$ with the quotient topology. Moreover, since $g$ is injective, $g$ is a homeomorphism since $g$ is an open map: given an open set $U\subset B_{m}$ we have $U = g^{-1}(g(U))$ and by the quotient topology on $O_{m}$, we have that $V = g(U)$ is open in $O_{m}$ since $g^{-1}(V)=U$ is open in $B_{m}$. Since $\{B_{m}: m\in M\}$ is a chart for the manifold structure of $M$, there exists continuous maps $\{\Phi_{m}: m\in M\}$ so that $\{(B_{m},\Phi_{m}):m\in M\}$ is an atlas for $M$. Therefore, the composition $\Phi_{m}\circ g^{-1}$ is a homeomorphism from $O_{m}$ to an open set $W$ in $\mathbb{R}^{K}$. Since $O_{m}$ is now covered by the single chart $\{O_{m},\Phi_{m}\circ g^{-1}\}$, we can pullback this smooth to $O_{m}$ via $\Phi_{m}\circ g^{-1}$ and endow $O_{m}$ with a smooth structure. Moreover, $B_{m}$ is homeomorphic to $O_{m}$ by construction.
\end{proof}

We make a few preliminary observations. First, by the homeomorphic relationship between $O_{m}$ and $B_{m}$, we conclude that $O_{m}$ is also a $K$-dimensional manifold. Also, $O_{m_{1}}$ is homeomorphic to $O_{m_{2}}$ for any $m_{1}, m_{2}\in M$. Moreover, the trivial element $\phi_{0}\in O_{m}$ for each $m\in M$ that maps $m$ to itself in $B_{m}$ acts a natural ``origin" in $O_{m}$.

We observe, again from the homeomorphic relationship between $B_{m}$ and $O_{m}$, that given any point $m_{1}\in B_{m}$ there is a unique operator $\phi_{1}$ such that $\phi_{1}(m)=m_{1}$. A question to ask is whether there is any relationship between $O_{m}$ and the tangent space $T_{m}$ at $m$. Indeed, for a smooth IAM, one can define tangent spaces $T_{m}$ at $m\in M$ and the OFM $O_{m}$ at $m$ is diffeomorphic to a neighborhood of $0\in T_{m}$. To see this fact, recall that the \emph{exponential map} defined on $T_{m}$ is a diffeomorphism between a neighborhood $U_{0}$ of $0\in T_{m}$ and a neighborhood of $V_{m}$ of $m$. We also have a diffeomorphism $\Phi$ from $O_{m}$ to $B_{m}$ by definition. Now, $B_{m}$ is open and there is an open ball $B_{0} \subseteqq U_{0}$ on which $\Phi^{-1} \circ \exp$ is a diffeomorphism, being a composition of two diffeomorphisms.

As a concrete example, consider again the translational manifold $M_{T}$ from Section \ref{sec:short}. Note that the parameter space in this case is $\Theta=\mathbb{R}^{2}$. Since there is no occlusion between any two images on $M_{T}$, it follows that given any $m,m'\in M_{T}$, there exists a $\phi\in O_{m}$ such that $\phi(m)=m'$. Since this is true for any pair of images in $M_{T}$, we conclude that $O_{m}=\mathbb{R}^{2}$ $\forall m\in M_{T}$ and hence, $B_{m}=M_{T}$ $\forall m\in M_{T}$. More generally, we note that for an IAM $M$ generated by Lie group actions without occlusion between images, the OFM $O_{m}$ at any point $m\in M$ can be identified with the parameter space $\Theta$ and neighborhood $B_{m}$ at each point is the entire manifold $M$. In particular, one can recover the geodesic path between two any two points $m,m' \in M$ by using appropriate flow operators in $O_{m}$ to generate the geodesic path from $m$ to $m'$. This shortest path corresponds to the geodesic in the parameter space between the parameter values corresponding to $m,m'$ as well.  We will return to this example in future sections and show how our more general formulation contains the algebraic methods such as \cite{grimson1,rao1,tuzel1,culpepper1} as special cases.

As in the case of the tangent bundle, we can construct an analogous bundle with the collection of $O_{m}$ since $m$ varies through $M$.

\begin{definition} Let $M$ be an IAM and $O_{m}$ the OFM at $m\in M$. The \emph{flow bundle} $OM$ on $M$ to be the disjoint union of the $O_{m}$ since $m$ varies over $M$ \begin{equation*} 
OM = \coprod_{m\in M} O_{m}.
\end{equation*}
\end{definition}

Thus, an element of $OM$ is a pair $(m,\phi)$ with $\phi\in O_{m}$. Using this fact, we can induce a topology on the flow bundle.

\begin{proposition}Given a $K$-dimensional IAM $M$, the flow bundle $OM$ is a $2K$-dimensional manifold.\end{proposition}

\begin{proof} We first note that if $(m,\phi) \in OM$ then $\{m,\phi(m)\} \in \{m\}\times B_{m}$. Thus, given an atlas $\{(U_{\lambda},\psi_{\lambda})\}_{\lambda\in\Lambda}$ of $M$, we have that \begin{equation*}\psi_{\alpha,\beta}^{*}((m,\phi)) = (\psi_{\alpha}(m),\psi_{\beta}(\phi(m)))\end{equation*} maps $(m,\phi)$ into $\psi_{\alpha}(U_{\alpha})\times\psi_{\beta}(U_{\beta}\bigcap B_{m})$ where $U_{\alpha}$ and $U_{\beta}$ are charts around $m$ and $\phi(m)$ respectively. The (continuous) inverse of $\{x_{1},...,x_{k},y_{1},...,y_{k}\} \in \psi_{\alpha}(U_{\alpha})\times\psi_{\beta}(U_{\beta}\bigcap B_{m})$ is given by $(m,\phi)$ where $\psi_{\alpha}(m) = \{x_{1},...,x_{k}\}$ and $\phi$ is the unique operator in $O_{m}$ such that $\psi_{\beta}(\phi(m)) = \{y_{1},...,y_{k}\}$ so that $OM$ is locally Euclidean.\end{proof}

Thus, we see that OFMs are manifolds consisting of flow operators that are defined pointwise on the corresponding IAM. The action of flow operators at a base point on the IAM results in motion \emph{along} the IAM, as opposed to linear transport that results in motion \emph{off} the manifold.

The key property that makes the study of OFMs interesting is that, for interesting IAMs, the associated OFMs are smooth and exhibit nice distance properties \cite{CVPROFM}. We summarize these in \ref{sec:smooth}. These two properties, namely smoothness and isometry, are in turn used to define a meaningful distance on the IAM. We discuss this next.

\subsection{Metric structure on an IAM via its OFMs}
Consider again the translational manifold $M_{T}$, where the OFM $O_{m}$ at each point $m\in M_{T}$ can be identified with $\mathbb{R}^{2}$. Being isometric with $\mathbb{R}^{2}$, we can endow each $O_{m}$ with the Euclidean metric which we denote as $d_{O}(\cdot,\cdot)$. Let $\theta_{1}, \theta_{2} \in \mathbb{R}^{2}$ be a pair of parameters such that $m_{1}=i(\theta_{1}), m_{2}=i(\theta_{2})$; note that there exists a $\phi\in O_{m_{1}}$ such that $\phi(m_{1})=m_{2}$. It then follows that $d_{O}(\phi_{0},\phi)=C\|\theta_{1}-\theta_{2}\|$ for some $C>0$.
As the  results of \cite{CVPROFM} indicate, the above discussion holds analogously for generic OFMs, i.e., each $O_{m}$ has an associated metric $d_{O}(\cdot,\cdot)$ and this metric is locally isometric to a corresponding metric on the parameter space $\Theta$. We indicate this as \begin{equation*}d_{O}(\phi_{0},\phi) \propto d_{\Theta}(\theta_{1},\theta_{2}),\end{equation*} where $\phi_{0}$ is the unique operator in $O_{m_{1}}$ such that $m_{1}=\phi_{0}(m_{1})$.

Our main focus in the remainder of this section is to define a corresponding metric for IAMs using the metric $d_{O}(\cdot,\cdot)$ on $O_{m}$. The resulting metric on $M$ inherits the property of being locally isometric to the changes in parameters. As a first step, we locally ``push forward" the metric on $O_{m}$ onto $B_{m}$ as follows. For points $m_{1},m_{2}\in M$ with $m_{2}\in B_{m_{1}}$, we have a unique operator $\phi_{1}$ such that $m_{2}=\phi_{1}(m_{1})$ so that we can define the distance $d_{M}(m_{1},m_{2})$ as the corresponding distance between $\phi_{0}$ and $\phi_{1}$
\begin{equation*}
d_{M}(m_{1},m_{2}) := d_{O}(\phi_{0},\phi_{1}).
\end{equation*}
Moreover, if $m_{1}=f(\theta_{1})$ and $m_{2}=f(\theta_{2})$ for parameters $\theta_{1}, \theta_{2} \in \Theta$ then we have
\begin{equation*}
d_{M}(m_{1},m_{2})\propto d_{\Theta}(\theta_{1},\theta_{2}).
\end{equation*} 
However, this definition does not readily extend to the case where $m_{1}$ and $m_{2}$ are not ``optically related", i.e., $m_{2}\notin B_{m_{1}}$. In this case, we first connect $m_{1}$ and $m_{2}$ by a path $c$ such that $c(0)=m_{1}, c(1)=m_{2}$. We then partition the domain of $c$ by a partition $P=\{0=t_{0}<t_{1}<\cdots<t_{n}=1\}$ such that the intermediate points along the path are optically related i.e., $c(t_{i})\in B_{c(t_{i-1})}$, where we assume that, for a fine enough partition, we can obtain such a nesting. We then define the distance along $c$ to be
\begin{equation*}
d(c,m_{1},m_{2})=\sup_{P}\displaystyle\sum\limits_{i=0}^{n-1} d_{M}(c(t_{i}),c(t_{i+1})),
\end{equation*}
where the supremum is over all partitions of the path.
By taking the infimum over all possible paths, we obtain a metric on $M$, i.e.,
\begin{equation}
d_{M}(m_{1},m_{2}) = \inf_{c} \displaystyle\sum\limits_{i=0}^{n-1} d_{M}(c(t_{i}),c(t_{i+1})).
\label{eq:metric4}
\end{equation}
In essence, the metric $d_{M}(\cdot,\cdot)$ is similar to the Riemannian distance: we first define distance over a fixed curve and then take the infimum over all possible paths (see Fig.\ \ref{fig:piece}).

\begin{proposition}
For an IAM $M$, the distance $d_{M}(\cdot,\cdot)$ in (\ref{eq:metric4}) is a metric on $M$.
\end{proposition}

\begin{proof} Positivity of $d_{M}(\cdot,\cdot)$ is clear, as is the fact that $d_{M}(m,m) =0$. If $m_{1} \neq m_{2}$, then for every path $c$ between $m_{1}$ and $m_{2}$, we have that $\displaystyle\sum\limits_{i=0}^{n-1} d_{M}(c(t_{i}),c(t_{i+1})) \neq 0$ and hence, $d_{M}(m_{1},m_{2}) \neq 0$. Symmetry follows from the fact that along $c$, $c(t_{i})\in B_{c(t_{i-1})}$ and hence,  $d_{M}(c(t_{i}),c(t_{i+1}))=d_{M}(c(t_{i+1}),c(t_{i}))$. For the triangle inequality, we note that given paths $c_{1}$ and $c_{2}$ from $m_{1}$ to $m_{2}$ and $m_{2}$ to $m_{3}$ respectively, the path $c_{1}*c_{2}$ obtained by traversing $c_{1}$ and $c_{2}$ in succession at twice the rate (i.e. $c_{1}*c_{2}(t) = c_{1}(2t)$ for $0\leqq t \leqq \frac{1}{2}$ and $c_{1}*c_{2}(t) = c_{2}(2t-1)$ for $\frac{1}{2}\leqq t \leqq 1$) is a path from $m_{1}$ to $m_{3}$ and $\displaystyle\sum\limits_{i=0}^{n-1} d_{M}(c_{1}(t_{i}),c_{1}(t_{i+1})) + \displaystyle\sum\limits_{i=0}^{n-1} d_{M}(c_{2}(t_{i}),c_{2}(t_{i+1})) \geqq \displaystyle\sum\limits_{i=0}^{n-1} d_{M}(c_{1}*c_{2}(t_{i}),c_{1}*c_{2}(t_{i+1}))$. By taking the infimum over all such paths, we verify the triangle inequality.\end{proof}

Recall that the metric on $O_{m}$ satisfies $d_{O}(\phi_{0},\phi) \propto d_{\Theta}(\theta_{1},\theta_{2})$ with $\phi \in O_{m}$ and $\theta_{1},\theta_{2}$ the parameters corresponding to $m$ and $\phi(m)$ respectively. From the above result, we see that, with the metric $d_{M}(\cdot,\cdot)$ on $M$, we have 
\begin{equation*}
d_{M}(m_{1},m_{2}) \propto d_{\Theta}(\theta_{1},\theta_{2}),
\end{equation*}
where $m_{1}$ and $m_{2}$ are points on $M$ corresponding to the parameter values $\theta_{1}$ and $\theta_{2}$ respectively under the assumption that the monotonicity of $d_{O}(\cdot,\cdot)$ over different $\{ O_{m} \}$ is universal, i.e., the distance between two points $m_1$ and $m_2$ does not change when we change our definitions of flow neighborhoods. 
We refer to this metric on $M$ as the \emph{flow metric}. We note that the flow metric is dependent on the OFMs $O_{m}$. We can now state the main result of this section.

\begin{figure}[t]
\centering
\includegraphics[width=0.7\textwidth]{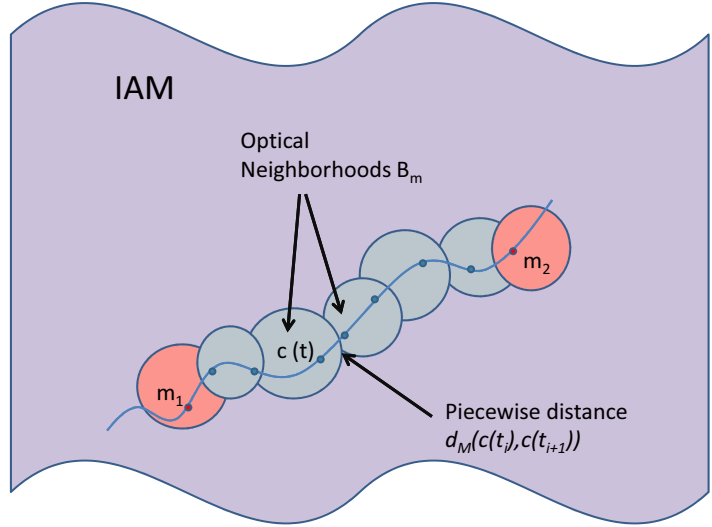}
\vspace{-4mm}
\caption{A pictorial representation of the OFMs and the flow metric for an IAM. Each OFM is mapped homeomorphically onto the flow neighborhood $B_{m}$ of the corresponding base point. The optical flow distance between two points is computed by considering the infimum over all possible curves of the piecewise flow distance $d_{M}(c(t_{i}),c(t_{i+1}))$.}
\label{fig:piece}
\end{figure}

\begin{theorem} Let $M$ be an IAM and let $d_{M}(\cdot,\cdot)$ be the associated flow metric. Then, $d_{M}(m_{1},m_{2}) \propto d_{\Theta}(\theta_{1},\theta_{2})$ as defined in (\ref{eq:metric4}) with $m_{1}$ and $m_{2}$  points on $M$ corresponding to the parameter values $\theta_{1}$ and $\theta_{2}$ respectively.\end{theorem}

The smoothness of the OFMs makes them amenable for use with conventional manifold processing tools. For instance, common dimensionality reduction methods such as \cite{ISOMAP,LLE,LLLE}, etc.\ assume that the manifold is smooth and, hence, are tools that are more appropriate for use with the flow metric. As an example, consider Fig.\ \ref{fig:OFMvsIAMISOMAP1}. Here, we see that the pairwise flow metric between points varies smoothly, as opposed to the Euclidean metric on the IAM, where due to sharp edges, the matrix of pairwise distance has large off-diagonal entries. Moreover, the residual error of non-linear dimensionality reduction using the flow metric decays very rapidly. As a result, it is more tractable to analyze IAMs using the  flow metric as opposed to the conventional Euclidean metric.

As an example, consider the problem of estimating $\theta\in \Theta$ such that $i(\theta)=m$ given a finite number of template points $m_{1},\cdots,m_{n}\in M$ with neighborhoods $B_{m_{1}},\cdots B_{m_{n}}$ that cover $M$. As a motivating special case, consider first the situation when $M=M_{T}$, the translational manifold with $\Theta=\mathbb{R}^{2}$ and a single template point $m'$. A similar problem has been dealt with in \cite{rao1,culpepper1} where the authors estimate the Lie group generators corresponding to the IAM. Given the parameter value $\theta'=(c_{1}',c_{2}')$ of the base image $m'$, we can compute the the parameter $\theta$ corresponding to $m$ as follows. First, we find the unique flow operator $\phi = (\phi_{1},\phi_{2})\in \mathbb{R}^{2}$ such that $\phi(m')=m$. The parameter value $\theta$ corresponding to $m$ is then obtained as $\theta = (c_{1}'+\phi_{1},c_{2}'+\phi_{2})$.

\sloppy
An entirely similar result holds for generic IAMs generated by Lie group actions. From this simple example, we see that finding the optimal $\theta$ is equivalent to finding the optimal flow operator $\phi$ that minimizes $d_{M}(\phi(m'),m)$ with $\phi\in O_{m'}$.

We return to the general case with $n$ template images and seek the optimal flow operator $\phi\in O_{m_{i}}$ that minimizes $d_{M}(\phi(m_{i}),m)$, $i=1,\cdots,n$. Here, a single neighborhood does not cover the entire IAM, and, hence, to estimate $\phi$ we first find the neighborhood $B_{\tilde{m}}\in\{B_{m_{1}},\cdots B_{m_{n}}\}$ such that $\tilde{m} = \arg\min_{m'\in \{m_{1},\cdots,m_{n}\}}d_{M}(m,m').$ Our search is then restricted to the neighborhood $B_{\tilde{m}}$. Within this neighborhood, we find the optimal $\phi\in O_{\tilde{m}}$ as above, i.e., $\phi = \arg\min_{\tilde{\phi}\in O_{\tilde{m}}}d_{M}(\tilde{\phi}(\tilde{m}),m)$. In essence, we first find the optimal template point and then search within the corresponding OFM for the optimal flow operator. If a single template point generates the entire IAM, then this 
procedure clearly reduces to the Lie group case discussed earlier. We thus see that the generic OFM formulation includes the algebraic methods of \cite{rao1,tuzel1} as a special case.

The remainder of the paper is devoted to developing geometric tools for IAMs that leverage the flow metric. Note that, unlike the tangent space, the OFM has no linear structure and, hence, we do not have at our immediate disposal tools such as parallel translation, covariant derivatives etc. We will construct analogous tools for our purposes via the flow metric and hence open up a vista for IAM analysis. 

%% file: difftools.tex
In this section, we develop the basic tools needed to analyze the structure of IAMs using the flow metric. We will pay special attention to flow operators defined along curves on the corresponding IAMs. Keeping computations in mind, these tools will open the door to a variety of applications as mentioned in the Introduction and beyond.

\subsection{Flow radius}
We first seek an appropriate measure of the size of an OFM and the corresponding flow neighborhood. In classical differential geometry, one measures the radius at a point in terms of the injectivity radius \cite{docarmo} using the Riemannian metric. In a similar fashion, we will measure the radius of a point $m$ in an IAM $M$ in terms of the flow metric.

\subsubsection{Flow Radius at $m\in M$}
\begin{definition} Given an IAM $M$ and $m\in M$ define the \emph{flow radius} or simply the \emph{radius} $r_{m}$ at $m$ as
\begin{equation*}r_{m} = \sup_{n\in B_{m}} d_{M}(m,n).\end{equation*}
\end{definition}
Note that we may regard $r_{m}$ as a function from $M$ to $\mathbb{R}^{+}$ i.e. $r(m) = r_{m}$ is a map from $M$ into the non-negative reals. Moreover, it is continuous as a function of $m$. Consider the variations of $r_{m}$ as $m$ varies. If $r_{m}$ is large, then one can find a suitable operator $\phi\in O_{m}$ that transports $m$ to a far away point, with distance measured using the flow metric. Conversely, a small $r_{m}$ indicates that $m$ can only be transported within a small region, or said differently, $m$ obstructs transport on the IAM. Moreover, rapid changes in the magnitude of $r_{m}$ within a small neighborhood indicate that the manifold is not well-behaved near $m$. In particular, this indicates that there are several points close to $m$ that obstruct transport while there are also several points that allow flow over large distances on the IAM.

\subsubsection{Flow radius for Lie groups}

A class of IAMs for which $r_{m}$ is very well behaved, indeed for which $r_{m}$ is a constant, are those generated by a Lie group action. As a motivating example, consider again the translational manifold $M_{T}$. We note that as $O_{m}=\mathbb{R}^{2}$ for each $m\in M_{T}$, it follows that $r_{m}=\infty$. If the parameter space $\Theta$ is compact, for instance, if $\Theta = S^{1}$ and we consider affine rotations of a base image $m$ generating the IAM $M$, then we again have that $B_{m}=M$ and $r_{m}$ is a constant $0<c<\infty$. For a generic IAM $M$ generated by Lie group actions with $B_{m}=M$, it follows that the flow radius is a constant whose exact value depends both on the object being imaged as well as the nature of the articulation.

\subsubsection{Flow curvature}
The above discussion indicates that the reciprocal $\frac{1}{r_{m}}$ is a measure of ``curvature". When $r_{m}$ is large (or infinite), then the IAM is can be thought as being ``optically flat" at $m$ in the sense that there is no obstruction to transport on the manifold at $m$. In the other limiting case, i.e.,  if $r_{m}$ approaches zero, the only operator in $O_{m}$ is the trivial operator $\phi_{0}$ and hence, $M$ is has high ``curvature" at $m$. We thus have the following definition.

\begin{definition} Given an IAM $M$ and $m\in M$  define the \emph{flow curvature} or simply the \emph{curvature} at $m$ as $K_{m} = \frac{1}{r_{m}}.$\end{definition}While the traditional notion of curvature is a point property, the flow curvature depends on both the base point as well as its neighborhood properties. In terms of flow curvature, we can now state that if $M$ is generated by a Lie group then $M$ has constant curvature. The class of IAMs with constant curvature will play a prominent role in our later analysis.

\subsection{Optical flow fields}
In this section, we  focus on a construct motivated by differential geometry, namely the idea of a vector field on a manifold. Recall that a vector field is a \emph{section} of the the tangent bundle, i.e., a vector field is a map $\sigma: M\rightarrow TM$ such that $\pi\circ\sigma = id_{M}$, where $\pi$ is the natural projection from the tangent bundle and $id_{M}$ is the identity map on $M$. In an analogous fashion, we define an \emph{optical flow field}, or simply a \emph{flow field}, as a section of the \emph{optical flow bundle} of $M$.

While vector fields are defined generically on manifolds, the special class of vector fields along curves is especially important in differential geometry. Vector fields along curves give rise to tools such as parallel translation, Jacobi fields, etc. \cite{docarmo}. In the case of IAMs, which lack analytic structure in general, we will define optical fields along curves on an IAM and recover similar geometric tools using the flow metric. As we assign operators to points on the curve, we intuitively would like the transport induced by the operator to remain along the curve so that the collection of flow operators induces {motion} along the curve. This special class of optical fields will be our main object of study for the rest of the paper.

\begin{definition} Let $M$ be an IAM and let $c$ be a smooth curve passing through $m_{1},m_{2}\in M$. Define $B_{c(t)}\bigcap c = \{c(t')\in O_{c(t)}: t'\geq t\}$. An \emph{optical field from $m_{1}$ to $m_{2}$ along $c$} with $m_{1},m_{2}\notin\partial c$ is a map $V:t\mapsto {OM}$ such that $V(t):=V_{t}\in O_{c(t)}$ and $V_{t}(c(t))\in \overline{B_{c(t)}\bigcap c}$ where by $V_{t}(c(t))$ we mean the point on the curve obtained by the action of $V_{t}$ on $c(t)$.\end{definition}

In other words, an optical field along $c$ is an assignment of a flow operator $V_{t}$ with $V_{t}$ an element of the OFM at $c(t)$ such that such that the action $V_{t}$ on $c(t)$ remains on the curve. Thus, the action of $V_{t}$ on the base point $c(t)$ at time $t$ induces motion along the curve. When $m_{1}$ and $m_{2}$ are clear from the context, we simply refer to $V_{t}$ as the optical field along $c$.

In essence, the curve $V_{t}$ traces a curve in ${OM}$ as $t$ varies with a consistent action on $c(t)$. By $B_{c(t)}\bigcap c$ we mean the intersection of the $K$-dimensional flow neighborhood $B_{c(t)}$ with the one dimensional curve $c$ starting at $c(t)$ and hence, $B_{c(t)}\bigcap c$ is a one dimensional embedded curve in $B_{c(t)}$ i.e., $B_{c(t)}\bigcap c$ is a one dimensional ``slice" of $B_{c(t)}$. We assume that $B_{c(t)}\bigcap c$ is connected.

To measure the distance traveled by the action of $V_{t}$ on $c(t)$, we define a radius $r_t$ restricted along the curve as opposed to the complete flow radius $r_{c(t)}$ at $c(t)$ \begin{equation*}r_{t} = \sup_{n\in B_{c(t)}\bigcap c} d_{c}(c(t),n),\end{equation*}where $d_{c}(\cdot,\cdot)$ is the flow metric restricted to the curve $c$. Likewise, the curvature $K_{t}$ along $c$ is the ratio \begin{equation*}K_{t}=\frac{1}{r_{t}}.\end{equation*}Since $B_{c(t)}\bigcap c$ is connected and 1D, the distance $d_{c}(c(t),m)$ between $c(t)$ and $m\in B_{c(t)}\bigcap c$ characterizes $m$ in the following sense. Given any positive constant $0 \leq \eta \leq r_{t}$ there is a unique $m\in B_{c(t)}\bigcap c$ with $d_{c}(c(t),m)=\eta$. Thus, by specifying the distance along the curve $c$, we effectively characterize the curve. Lifting this observation into $O_{c(t)},$ we have the following result.

\begin{theorem} Let $M$ be an IAM, let $c$ be a curve passing through $m_{1},m_{2}\in M$, and let $V_{t}$ be an optical field from $m_{1}$ to $m_{2}$ along $c$. Then, $V_{t}$ is completely characterized by the function $h_{V}(t) = d_{c}(V_{t}(c(t)),c(t))$ in the sense that for any non-negative function $h(t)$ bounded pointwise by $r_{t}$, there exists a unique optical optical field $V_{t}$ such that $h_{V}(t)=h(t).$\end{theorem}\label{}
\begin{proof}That $0\leq h_{V}(t)\leq r_{t}$ is clear from the definitions above. Let $h(t)$ be any non-negative function bounded by $r_{t}$. Then, for fixed $t=t_{0}$ we have that $0 \leq h(t_{0}) < r_{t_{0}}$. By the remark made previously, we have a unique $m_{t_{0}}\in B_{c(t_{0})}\bigcap c$ with $d_{c}(c(t_{0}),m_{t_{0}})=h(t_{0})$. Now, as $m_{t_{0}} \in B_{c(t_{0})}\bigcap c$, in particular, $m_{t_{0}}\in B_{c(t_{0})}$ and hence there exists a unique $\phi_{t_{0}}\in O_{c(t_{0})}$ such that \begin{equation*}\phi_{t_{0}}(c(t_{0})) = m_{t_{0}}.\end{equation*} As $t_{0}$ was arbitrary, as $t$ varies, we can define $V_{t} =  \phi_{t}$. Moreover, \begin{equation*}h_{V}(t) = d_{c}(V_{t}(c(t)),c(t)) = d_{c}(\phi_{t}(c(t)),c(t)) =  d_{c}(m_{t},c(t)) = h(t)\end{equation*} so that $h(t)$ characterizes $V_{t}$.\end{proof}Recall that we view an optical flow field $V_{t}$ along $c$ as a curve in $OM$. Theorem 10 states that this curve is characterized by the function $h_{V}(t)$ in the sense that for a given $V_{t}$ along the (fixed) curve $c$, the function $h_{V}(t)$ contains all the information about the motion induced by $V_{t}$ on the curve i.e., $h_{V}(t)$ measures the distance to which $V_{t}$ transports $c(t)$ in $B_{c(t)}\bigcap c$. Since this function is of prime importance, we make the following definition.

\begin{definition} Given an optical field $V_{t}$ along a curve $c$, define its \emph{motion function} to be \begin{equation*}h_{V}(t) = d_{c}(V_{t}(c(t)),c(t)).\end{equation*}\end{definition}Note that by continuity of the metric $d_{c}(\cdot,\cdot)$ and $c(t)$, the function $h_{V}(t)$ is also continuous. 

\subsection{Parallel flow fields}
There is a very natural geometric interpretation of the motion function $h_{V}$ of an optical flow field $V_{t}$ along a curve $c$. Namely, it is a measure of the distance traveled along the curve at time $t$ by $c(t)$ when acted upon by $V_{t}$. Thus, it is natural to think of instantaneous changes in $h_{V}$ in $t$ as a measure of the \emph{velocity} of the motion induced by $V_{t}$ on $c(t)$. In classical geometry, the class of constant velocity curves is especially important; they correspond to uniform motions. Similarly, we will be interested in the class of constant motion functions $h_{V}$. These correspond to optical fields $V_{t}$ along $c$ that induce uniform motion along $c$, where by uniform motion we mean the distance traveled along $c$ is constant for all time.

Another key link with classical differential geometry is the notion of parallel transport (or parallel translation) of tangent vectors \cite{docarmo}. Parallel transport is a key tool for ``moving'' tangent vectors from different tangent spaces while preserving direction/orientation along the curve. Parallel transport utilizes the inherent linear structure of the tangent space to define a linear map between tangent on points along the curve. We aim to develop a similar analytic tool for IAMs. However, an immediate stumbling block is the clear lack of \emph{linear structure} in the OFM. We therefore take a different approach to defining parallel translation in the OFM case using motion functions.

To motivate our definition, we recall that a vector field along a curve $c$ on a manifold is \emph{parallel} if its \emph{covariant derivative} along $c$ vanishes. The covariant derivative is in essence a way of differentiating the vector field along $c$. The relation between parallel translation and parallel vector fields is that, given a tangent vector $v$ in the tangent space of a point $c(t)$ on the curve $c$, it is possible to extend $v$ along $c$ by parallel translation to yield a parallel vector field along $c$. In the OFM case, we have the motion function of an optical field along $c$ at our disposal and we use it to characterized parallelism of the field along the curve.

\begin{definition} 
An optical field $V_{t}$ along a curve $c$, is {\em parallel} if the derivative of its motion function with respect to $t$ is zero, or equivalently, if the motion function is constant along $c$, i.e., $h_{V}(t)$ is a constant.
\end{definition}
In all that follows, we will denote by $\Omega(c,m_{1},m_{2})$ the space of all optical fields through $m_{1}$ and $m_{2}$ along a curve $c$ passing through $m_{1},m_{2}\in M$. When $m_{1},m_{2}$ are understood from context, we simply refer to this space as $\Omega(c)$. The subclass of parallel fields along $c$ will be denoted by $\omega(c)$.

A few facts are immediate from the above definition. First, since $h_{V}(t)\leq r_{t}$ for all $t$, it is clear that if $V_{t}$ is to be parallel along $c$, then $h_{V}(t)$ is a constant $h_{V}$ independent of $t$ and $h_{V} \leq \inf_{t} r_{t}$. Therefore, in the future, we will suppress the argument $t$ in the motion function $h_{V}(t)$ of a parallel flow field. Second, since $V_{t}$ is characterized by $h_{V}$, we see that for any constant $\delta$ such that $0\leq \delta < \inf_{t} r_{t}$, there is a parallel optical field along $c(t)$ such that $h_{V}=\delta$. Such a parallel optical field can be obtained, for example, by choosing for each $t$ a flow operator $\phi_{t} \in O_{c(t)}$ with $d_{c}(\phi_{t}(c(t)),c(t))=\delta$. The existence of such an optical field $\phi_{t}$ is guaranteed by the above theorem.
Given $\phi \in O_{c_{0}}$ we aim to extend $\phi$ throughout the curve to obtain a flow field $V_{t}$ such that $V_{t}$ is parallel along $c$ with $h_{V}=d_{c}(\phi(c_{0}),(c_{0}))$. The following result allows for such parallel translation of flow operators along a curve.

\begin{proposition}
Let $M$ be an IAM, and let $c$ be a curve through $m_{1},m_{2}\in M$. Let $\epsilon = \inf_{t} r_{t}$. Then, given any $\phi\in O_{c(t_{0})}$ with $\delta = d_{c}(\phi(c(t_{0})),c(t_{0})) < \epsilon$ there exists a unique parallel optical field along $c(t)$ with $h_{V}(t) = \delta.$
\end{proposition}
\begin{proof}
As $\delta < \epsilon$, in particular $\delta < r_{t}$. Thus, invoking the previous theorem for the special case of the constant function $h(t)=\delta$, we have the existence of a unique optical field $V_{t}$ along $c(t)$ such that $h_{V}(t) = \delta$. Moreover, since $h_{V}(t)$ is constant, $V_{t}$ is parallel. 
\end{proof}
In contrast with the classical case, parallel transport along a curve is dependent on the nature of the flow operator, i.e., an arbitrary flow operator $\phi \in O_{c_{0}}$ cannot be parallel translated along $c$ unless $d_{c}(\phi(c_{0}),c_{0})<\inf_{t} r_{t}$. This constraint is related to the nature of the curve; parallel transport along a curve that contains points with high curvature $K_{t}$ is limited to those flow operators that induce smaller motion along the curve. Moreover, the possibility of parallel transport of $\phi \in O_{c(t)}$ has a global dependence, i.e., it depends on the curvature of the entire curve, not only the curvature $K_{t}$ at the point $c(t)$.

Thus, we are naturally led to study those curves for which parallel translation of an operator $\phi$ at a single point $c(t_{0})$ ensures the existence of parallel translation of operators at any other point on the curve. Clearly, the necessary condition is the invariance of $K_{t}$ with $t$, and therefore we are led to consider curves for which the curvature $K_{t}$ is independent of $t$, i.e., a constant. This special class of curves has a very rich structure that we explore in the following sections.

\subsection{Approximation of arbitrary flow fields by parallel flow fields}
In this section, we consider the problem of approximating an arbitrary $V_{t}\in\Omega(c)$ by elements in $\omega(c)$. Consider an optical field $V_{t}\in\Omega(c)$ (not necessarily parallel) along a curve $c$ of constant flow curvature. A natural question to ask is how far away $V_{t}$ is from being parallel. One way to do this is to seek the ``best'' approximation of $V_{t}$ by a parallel field $W_{t}\in\omega(c)$ along $c$. To quantify the approximation, we consider the following cost 
\begin{equation*}
e(t) = d_{c}(V_{t}(c(t)),W_{t}(c(t))).
\end{equation*}
In words, $e(t)$ is a measure of how the action of $V_{t}$ on $c(t)$ differs from the action of $W_{t}$ on $c(t)$. Note that $e(t)$ is bounded above by $h_{V}(t)+h_{W}$ since $e(t)=d_{c}(V_{t}(c(t)),W_{t}(c(t))) \leq d_{c}(V_{t}(c(t)),c(t))+d_{c}(c(t),W_{t}(c(t)))=h_{V}(t)+h_{W}$.

While $e(t)$ is a pointwise error, we will need to consider the total error over the entire curve. To this end, a natural choice of error metric is
\begin{equation*}E(V,W) = \displaystyle\int^b_a e(t)dt,\end{equation*}
where the domain of $c$ is the interval $(a,b)$.
Thus, our goal is to find a parallel optical field $W_{t}\in\omega(c)$ that minimizes $E(V,W)$, i.e., \begin{equation*}W^{*} = \arg\min_{W} E(V,W).\end{equation*}
In general, a minimizer may not exist, or it may not be unique if one exists. However, the greatest issue is the strong dependence of the error $e(t)$ on the flow metric, which prevents a generic solution to the minimization problem
since we cannot infer the convexity of the problem as stated. We can, however, obtain a universal lower bound on the error independent of the flow metric as follows
\begin{equation*}|h_{V}(t)-h_{W}| = |d_{c}(V_{t}(c(t)),c(t))-d_{c}(W(c(t)),c(t))|\leq e(t).\end{equation*}
Therefore, we seek a minimizer of $\displaystyle\int^b_a |h_{V}(t)-h_{W}|dt$. Note that since $W_{t}$ is parallel, $h_{W}$ is a constant, say $h\geq0$. Moreover, since $c$ is a path of constant flow curvature, $r_{t}$ is a constant $r>0$. Since $h$ characterizes $W_{t}$, a minimizer $h^{*}$ yields a lower bound for $E(V,W)$. Our goal then, is to find an optimal constant $h^{*}$ that minimizes $\widetilde{E(h)} = \displaystyle\int^b_a |h_{V}(t)-h|dt$. We claim that a solution $h^{*}$ is a certain ``median" of $h_{V}(t)$.
We first define $A_{k} = \{t\in(a,b): h_{V}(t)>k\}$ and $B_{k} = \{t\in(a,b): h_{V}(t) \leq k\}$ for some $k$. We claim that an optimal constant $\hat{h}$ is such that $\lambda(A_{\hat{h}})=\lambda(B_{\hat{h}})$, where $\lambda(S)$ denotes the measure of a set $S$.

\begin{theorem}
Let $\widehat{h}$ be the constant such that $\lambda(A_{\widehat{h}})=\lambda(B_{\widehat{h}})$. Then, $\widehat{h}$ minimizes $\widetilde{E(h)}$.
\end{theorem}

\begin{proof}
Note that the function $\widetilde{E(h)}$ is convex in $h$ with $h\in(a,b)$ and $(a,b)$ a convex interval. Therefore, we are guaranteed a minimizer $h^{*}$. Now, without loss of generality, we assume that $h^{*}<\widehat{h}$. We evaluate the cost over the two regions $A_{\widehat{h}}$ and $B_{\widehat{h}}$ \begin{equation*}\widetilde{E(h^{*})}=\displaystyle\int^b_a |h_{V}(t)-h^{*}|dt =\displaystyle\int_{A_{\widehat{h}}} |h_{V}(t)-h^{*}|dt+\displaystyle\int_{B_{\widehat{h}}} |h_{V}(t)-h^{*}|dt.\end{equation*}

Define $B_{1}=\{t:h^{*}<h_{V}(t)<\hat{h}\}$ and $B_{2}=\{t:h_{V}(t)<h^{*}\}$ and note that $\lambda(A_{\widehat{h}})-\lambda({B_{1}})-\lambda({B_{2}})=0$. We can now express $\widetilde{E(h^{*})}$ in terms of $\widetilde{E(\widehat{h})}$ as
\begin{equation*}
\widetilde{E(h^{*})} = \widetilde{E_{A_{\widehat{h}}}}(\widehat{h})+ (\widehat{h}-h^{*})\lambda(A_{\widehat{h}})+\widetilde{E_{B_{\widehat{h}}}}(\widehat{h})-(\widehat{h}-h^{*})\lambda(B_{2})-\alpha,\end{equation*}
where $\widetilde{E_{X}}(\widehat{h})$ denotes the cost function restricted to the subset $X\subseteq(a,b)$ and $\alpha$ is a positive constant that measures the difference of $\widehat{h}-h^{*}$ on the set $B_{1}$. The maximum value of $\alpha$ is $(\widehat{h}-h^{*})\lambda({B_{1}})$ and hence \begin{equation*}\widetilde{E(h^{*})}\leq\widetilde{E_{A_{\widehat{h}}}}(\widehat{h})+ (\widehat{h}-h^{*})(\lambda(A_{\widehat{h}})-\lambda({B_{1}})-\lambda({B_{2}}))+\widetilde{E_{B_{\widehat{h}}}}(\widehat{h}).\end{equation*} 
Since $\lambda(A_{\widehat{h}})-\lambda({B_{1}})-\lambda({B_{2}})=0$, we conclude that \begin{equation*}\widetilde{E(h^{*})}\leq\widetilde{E(\widehat{h})}\end{equation*} and hence $h^{*}=\widehat{h}$. The case of $h^{*}>\widehat{h}$ follows from symmetry.\end{proof}

We illustrate the approximation of non-parallel flow fields by parallel flow fields with the example of \emph{video resampling} where we can consider a video $\mathcal{I}$ to be a curve on an IAM i.e., a video $\mathcal{I}=\{I_{t}, 0\leq t\leq T\}$ with $T>0$. This application is related to the problem of dynamic time warping (DTW) \cite{dtw1,dtw2,dtw3}, where one is interested in measuring the similarity between two sequences that vary in time or speed. As we shall see, this can be used for matching or aligning video sequences with a warped time axis \cite{dtw2,dtw3}.

Consider the IAM generated by imaging a black disk on an infinite white background starting with an initial velocity $v_{0}$ and accelerating with constant acceleration $a$ along a fixed direction. For instance, the disk can be thought of as undergoing freefall off an infinitely high cliff. The IAM is a 1D curve $c$ and homeomorphic to $\mathbb{R}^{+}$, the non-negative reals. Note that the curvature is everywhere zero since $r_{t}=\infty$. Given an arbitrary flow field $V_{t}$ along $c$, our goal is to analytically construct a parallel field $\tilde{V}_{t}$ such that $\tilde{V}_{t}$ is the unique parallel flow field that minimizes $\widetilde{E(h)}$. We first consider the video obtained by the action of $V_{t}$ i.e. $\mathcal{I}=\{V_{t}(c(t))\}$. Since $V_{t}$ is not parallel, the video will show the disk moving with non-uniform motion. Our goal is to make the video uniform, i.e., generate a new video $\tilde{\mathcal{I}}$ from $\tilde{V_{t}}$ that shows the disk moving with uniform motion.

From the physics of the problem, it is clear that $d_{c}(c(t),c(t+\delta_{t})) = K(v_{t}\delta_{t}+\frac{1}{2}a\delta^{2}_{t})$ for some positive constant $K$ and any time increment $\delta_{t}$ with $v_{t}$ being the velocity at time $t$. Now, since $\tilde{V}_{t}$ is to be a parallel flow field, we have that $h_{\tilde{V}}$ is a constant denoted by $h$. Thus, $d_{c}(\tilde{V}_{t}(c(t)),c(t))=d_{c}(c(t),c(t+\delta_{t}))=h$ so that $K(v_{t}\delta_{t}+\frac{1}{2}a\delta^{2}_{t}) = h$. Rearranging this equation, we arrive at 
\begin{equation*} \delta^{2}_{t}+\frac{2v_{t}}{a}\delta_{t}-\frac{2h}{aK}=0.\end{equation*}
Solving for $\delta_{t}$, we obtain two real roots $\delta^{1,2}_{t}=\frac{-v_{t}}{a}\pm\sqrt{(\frac{v_{t}}{a})^{2}+\frac{2h}{aK}}$ of which the (physically meaningful) positive root $\delta^{1}_{t}=\frac{-v_{t}}{a}+\sqrt{(\frac{v_{t}}{a})^{2}+\frac{2h}{aK}}>0$ is retained. Thus, by defining $\tilde{V}_{t}$ such that $\tilde{V}_{t}(c(t))=c(t+\delta^{1}_{t})$, we see that $\tilde{V}_{t}$ is the unique parallel flow field that minimizes $\widetilde{E(h)}$. The new video $\tilde{\mathcal{I}}=\{\tilde{V}_{t}(c(t))\}$ will thus show the disk moving with constant velocity. With this, we have effectively linearized the motion and made it independent of the acceleration of the disk.

%% file: multiscale.tex
As indicated in the Section \ref{sec:difftools}, the set of parallel optical fields is a very special subset of the set of all optical fields along a fixed curve $c$. In this section, we will construct a monoid structure (i.e., a set with an associative operation and identity) on the set of all optical fields along a {fixed} curve $c$ and show that the class of {parallel fields} forms a submonoid of this set under some conditions on the curvature along the curve. Moreover, the monoid operation yields a multiscale structure on the set of parallel optical fields.

\subsection{Monoid structure on $\Omega(c)$}
As noted previously in Section \ref{sec:theory}, a clear disadvantage in dealing with the space of all optical fields along a curve is the lack of a linear, or more generally, any algebraic structure. In order to remedy this situation, we will define a binary operation on the set of optical fields that yields a monoid structure. We first fix a curve $c$ passing through $m_{1},m_{2}\in M$.

Recall that a generic $V_{t}\in \Omega(c)$ is characterized by its motion function $h_{V}(t)$. Thus, operations defined on motion functions $h_{V}(t)$ translate to operations on $V_{t}\in\Omega(c)$. With this in mind, we define for $V_{t},W_{t}\in \Omega(c) $ the sum $V_{t}+W_{t}$ to be the unique optical field with motion function 
\begin{equation*}
h_{V+W}(t) = \min(h_{V}(t)+h_{W}(t),r_{t}).
\end{equation*}
Since $h_{V+W}(t) \leq r_{t}$ for all $t$, we see that $h_{V+W}(t)$ corresponds to a unique flow field that we define to be $V_{t}+W_{t}$. This operation is clearly commutative. Note also that the trivial (parallel) field $Z_{t}$ defined to be the field that acts trivially on $c(t)$; i.e., \begin{equation*}
Z_{t}(c(t)) = c(t)
\label{ZeroElement}
\end{equation*}
is characterized by the motion function $h_{Z}(t)=0$ since $h_{Z}(t)=d_{c}(Z_{t}(c(t),c(t))=d_{c}(c(t),c(t))=0$. Moreover, for any $V_{t}\in\Omega(c)$, we have that 
\begin{equation*}
Z_{t}+V_{t} = V_{t}.
\end{equation*}
We see that $Z_{t}$ acts as the identity element in $\Omega(c)$. In addition to this, there is also the unique optical field $U_{t}$ characterized by \begin{equation*}h_{U}(t) = r_{t}\end{equation*} that satisfies 
\begin{equation*}
V_{t}+U_{t} = U_{t}
\end{equation*}
for all $V_{t}\in \Omega(c)$ and hence acts as the ``absorbing" element of $\Omega(c)$.

Clearly, we do not have ``inverses" with respect to ``+" in the sense that, given a generic $V_{t}\in \Omega(c)$, there does not exist a $W_{t}$ such that $V_{t}+W_{t}=Z_{t}$. However, we do have ``conjugates" with respect to ``+" in the following sense. Given any $V_{t}\in \Omega(c)$, there is a unique $V_{t}^{*}\in \Omega(c)$ such that 
\begin{equation*}
V_{t}+V_{t}^{*} = U_{t}.
\end{equation*}
$V_{t}^{*}$ is defined by its motion function 
\begin{equation*}
h_{V^{*}}(t) = r_{t}-h_{V}(t).
\end{equation*}
We refer to $V^{*}_{t}$ as the conjugate of $V_{t}$. Moreover, 
\begin{equation*}
(V_{t}^{*})^{*} = V_{t}.
\end{equation*} 
Finally, we note that $U_{t}$ and $Z_{t}$ are conjugates.

\begin{proposition}$\Omega(c)$ is a monoid under the operation ``+" defined in  (22).\end{proposition}

\begin{proof}
We only verify associativity, since $Z_{t}$ provides the identity. Given $V_{t},X_{t},Y_{t}\in\Omega(c)$, we consider the sums $(V_{t}+X_{t})+Y_{t}$ and $V_{t}+(X_{t}+Y_{t})$. If $h_{V}(t),h_{X}(t),h_{Y}(t)$ are such that $h_{V}(t)+h_{X}(t)+h_{Y}(t)< r_{t}$ then both sides are clearly equal. If on the other hand, $h_{V}(t)+h_{X}(t)+h_{Y}(t)\geq r_{t}$ then we must have that either the sum $h_{V}(t)+h_{X}(t)+h_{Y}(t)$ taken two factors at a time exceeds $r_{t}$ or that the combined sum of all three factors exceeds $r_{t}$ with the sum of no two factors exceeding $r_{t}$. In the first case, we assume that $h_{V}(t)+h_{X}(t)\geq r_{t}$, which implies the sum
\begin{equation*}
(V_{t}+X_{t})+Y_{t}=U_{t}
\end{equation*}
with $(X_{t}+Y_{t})\neq U_{t}$, i.e., 
\begin{equation*}
h_{X}(t)+h_{Y}(t)<r_{t}.
\end{equation*}
However, $V_{t}+(X_{t}+Y_{t})$ is characterized by 
\begin{equation*}
\min(h_{V}(t)+\min(h_{X}(t)+h_{Y}(t),r_{t}),r_{t})= \min(h_{V}(t)+h_{X}(t)+h_{Y}(t),r_{t}) = r_{t}
\end{equation*} 
which shows that 
\begin{equation*}
V_{t}+(X_{t}+Y_{t})=U_{t}
\end{equation*}
as well. The other cases follow by similar arguments.\end{proof}

\subsection{Parallel fields along curves of constant curvature}
The addition operation defined above restricts to the set $\omega(c)$ of parallel fields along $c(t)$. However, for generic $c(t)$, the sum $V_{t}+W_{t}$ of two parallel fields $V_{t},W_{t}\in \omega(c)$ may result in a possibly non-parallel field. For instance,
consider $h_{V}(t),h_{W}(t)$ such that $h_{V}(t_{0})+h_{W}(t_{0}) < r_{t_{0}}$ but $h_{V}(t_{1})+h_{W}(t_{1}) \geq r_{t_{1}}$ for some $t_{0},t_{1}$.
Clearly, the sum is not parallel since in the first case, $h_{V+W}(t_{0})=h_{V}(t_{0})+h_{W}(t_{0}) < r_{t_{0}}$ while $h_{V+W}(t_{1})=r_{t_{1}}$ so that $h_{V+W}(t)$ is not constant.

However, that if $r_{t}$ is constant along $c$, i.e., the flow radius is constant along the path, then the above situation is vacuous. Since $K_{t} = \frac{1}{r_{t}}$ is the curvature along the curve $c$, we are essentially requiring the curvature to be constant along $c$. We formally record the above observation.

\begin{proposition}
Let $c$ be a curve with constant curvature $K_{t}$. Then, the operation ``+" restricted to $\omega(c)$ is well defined  and $\omega(c)$ is a submonoid of $\Omega(c)$.
\end{proposition}
\begin{proof}
We need only verify closure of ``+" in $\omega(c)$ since $Z_{t}\in\omega(c)$. If $V_{t}\in \omega(c)$, then $h_{V}(t)$ is a constant and we will therefore suppress the argument $t$ in $h_{V}(t)$. Note that since $c$ is of constant curvature, $r_{t}$ is a constant $r>0$. Thus, if $V_{t},W_{t}\in \omega(c)$, then $h_{V+W} = \min(h_{V}+h_{W},r)$. If $h_{V}+h_{W}< r$ then as $h_{V}$ and $h_{W}$ are both constant so is their sum $h_{V}+h_{W}$ and 
\begin{equation*}
h_{V+W}=h_{V}+h_{W}.
\end{equation*} 
If $h_{V}+h_{W}\geq r$, then 
\begin{equation*}
h_{V+W}=\min(h_{V}+h_{W},r)=r.
\end{equation*} 
In either case, $h_{V+W}$ is constant and hence
\begin{equation*}
V_{t}+W_{t}\in \omega(c).
\end{equation*}
\end{proof}
Thus, curves $c$ that have constant curvature at all points are very special; not only is parallel transport determined by a single point on the curve, but $\omega(c)$ is also a submonoid of $\Omega(c)$.

\subsection{Multiscale structure of $\omega(c)$}
In this section, we look for finite submonoids of $\omega(c)$. We would like the submonoids to be canonically defined,
 by which we mean that they depend only on the geometry of the constant curvature curve $c$. In particular, we construct for each positive integer $k$, a collection of finite submonoids $\mathcal{V}_{n,k}$ of $\omega(c)$ with $n$ a non-negative integer and $c$ a constant curvature path. Moreover, we will see that this collection is naturally nested, i.e., \begin{center}$\mathcal{V}_{0,k}\subset\mathcal{V}_{1,k}\subset \cdots$\end{center}In essence, this provides a multi-scale view of $\omega(c)$.

Since $\omega(c)$ consists of parallel optical fields, we will simply denote the motion function $h_{V}(t)$ of $V_{t}\in \omega(c)$ as $h_{V}$. Moreover, we will suppress the subscript $t$ when referring to elements of $V_{t}\in\omega(c)$. Also, for $V\in\omega(c)$, we mean by $\frac{V}{k}$ the element in $\omega(c)$ with motion function being the constant $\frac{h_{V}}{k}$. Recall also that $\omega(c)$ possesses two canonical elements $Z,U\in\omega(c)\cap\Omega(c)$ that act as the trivial and absorbing elements of $\omega(c)$ respectively.

We begin now with the construction of $\mathcal{V}_{n,k}$. Fix a positive integer $k$ and set \begin{center}$\mathcal{V}_{0,k}=\{Z,U\}$.\end{center} Next, we inductively set \begin{center}$\mathcal{V}_{n,k} = \left\{ Z,\frac{U}{k^{n}},\frac{2U}{k^{n}},...\frac{(k^{n}-1)U}{k^{n}},U \right\}$.\end{center} Now, it is clear that \begin{center}$\mathcal{V}_{0,k}\subset\mathcal{V}_{1,k}\subset...$.\end{center} Moreover, by the constant curvature condition, the sum of any two elements in $\mathcal{V}_{n,k}$ remains in $\mathcal{V}_{n,k}$ while associativity is obtained from the corresponding property in $\omega(c)$. Finally, since $Z\in\mathcal{V}_{n,k}$ for all $n,k$, we conclude that each $\mathcal{V}_{n,k}$ is a finite submonoid of $\omega(c)$.

We have thus obtained a sequence of finite submonoids of $\omega(c)$. Moreover, with increasing $n$, it is clear that an arbitrary $V\in\omega(c)$ can be uniformly approximated by $\widetilde{V}\in\mathcal{V}_{n,k}$ in the sense that $|h_{V}-h_{\widetilde{V}}|$ can be made arbitrarily small by choosing larger $n$. In other words, the sequence $\mathcal{V}_{n,k}$ is ``dense" in $\omega(c)$. Finally, since the basic generators of $\mathcal{V}_{n,k}$ are $Z,U \in\omega(c)\cap\Omega(c)$, this construction is canonical in the sense that it depends only on the set of $\{Z,U\}$, which are in turn are characterized by the global geometry of the curve $c$.

In Section \ref{sec:difftools}, we saw that with curvature conditions on a curve $c$, it is possible to approximate an arbitrary flow field by a parallel one. The results presented above show that one can further approximate a parallel flow field with a finite collection of ``template'' elements from $\omega(c)$. This is very much similar to the multiscale representation that wavelets provide for natural images. Indeed, the multiscale structure inherent to monoids paves the way for ``lossy compression" of arbitrary flow fields by storing only the relevant scales $\mathcal{V}_{n,k}$. In addition to compression, the multiscale structure can potentially enable
fast computations on the flow fields. Operations on the flow fields can equivalently be mapped to those  on the $\mathcal{V}_{n,k}$ without loss in accuracy while gaining significantly in the number of computations required.

%% file: discuss.tex
In this paper, we have developed the mathematical foundations of optical flow-based transport operators for image manifolds, which were introduced empirically \cite{CVPROFM}.
Our main theoretical contribution was the development of the flow metric for using the ambient metric on the OFMs. Using the flow metric, we derived differential geometric analogues of tangent bundles, vector fields, parallel transport, curvature etc. When the IAM is generated by Lie group parameters, we showed that the OFM framework includes previous algebraic methods as a special case. Moreover, since the flow neighborhood at each point is the entire IAM, we can obtain geodesics between any two images using flow operators.
While this paper has focused on explaining and extending the results obtained in \cite{CVPROFM}, we envision that the theory could make an on a large class of applications involving image manifolds.  

A clear assumption in our analysis has been that optical flow is the transport operator of choice for IAMs. While this is true for a majority of IAMs generated by motion-induced parameter changes such as translations, rotations and unstructured plastic deformations, there are classes of IAMs for which the choice of transport operator is not immediately clear. For instance, for illumination manifolds obtained by variations in the illumination of an object, optical flow may not be the transport operator of choice, since such manifolds do not in general obey the brightness constancy requirement needed in optical flow computations. Moreover, in cases where there is significant self-occlusion during the imaging process, optical flow may not be a practical transport operator. However, by regularizing the optical flow computation to handle occlusions by removing flow operators that lead to undefined motion between pixels, one can partially circumvent this issue, and the methods of this paper can be profitably applied. However, these are aspects of a more computational nature, and we reserve them for an alternate forum.

A number of avenues for future work exist. First, our work hints that it should be possible to  develop additional geometric tools such as affine connections, holonomy, etc.\ using flow operators. These will complete the IAM analysis toolbox. Second, although our development has been specific to the case of IAMs with optical flow, the basic model is extensible to a wide variety of signal ensembles with appropriately defined transport operators. For instance, a manifold model for speech signals has been proposed in \cite{speechmanifold} where appropriate transport operators and the analog of OFMs may involve a frequency domain approach. Once the ``right" transport operator has been identified for the application in hand, one can conceivably define and study metrics similar to our flow metric and thereby develop analytic tools for further analysis. 

%% file: smoothproofs.tex
In this appendix, we establish the smoothness and isometry properties of OFMs associated with two interesting classes of IAMs: the affine articulation manifold and the pose manifold.

\subsection{Affine articulation manifold}
Affine articulations are parameterized by  a $6$-dimensional articulation space $\Theta = \reals^6$;
each articulation $\theta$ can be written as $\theta = (A, \bft)$, with  $A \in \reals^{2 \times 2}$ and ${\bf t} \in \reals^2$. Any image belonging to the IAM can thus be written as
\begin{equation}
I_1(\bfx) = I_\rf \left( (A+{\mathbb I})\bfx + \bft \right),
\label{eqn:trn}
\end{equation}
where $\bfx = (x, y) $ is defined over the domain ${\cal X} = [0,1] \times [0,1]$ and where $I_\rf$ is a reference image.\footnote{Boundary-related issues are always an concern when we define images over a finite domain. Here, we circumvent this by assuming that the regions of interest are surrounded by a field of zeros. This allows us to assign undefined values to zero and satisfy  (\ref{eqn:trn})}
The optical flow field $f_\theta$ associated with the transport operator can be written as 
$f_\theta(\bfx) = f_{A, \bft}(\bfx) = A \bfx + \bft.$
Now, recall that the OFM at $I_\rf$ is defined as
\begin{equation*}
O_{I_\rf} = \{ f_\theta: \;\; \theta = (A, \bft), A \in \reals^{2 \times 2}, \bft \in \reals^2 \}.
\end{equation*}
The linear dependence of the optical flow field on both $A$ and $\bft$ implies that the OFM, which is the collection of optical flows at $I_\rf$, is infinitely smooth. Finally, noting that the OFM is independent of the reference image, we can establish the following result.

\begin{appxlem} \label{lem:affsmooth}
For affine articulations, the OFM at any reference image is infinitely smooth.
\end{appxlem}

Next, we consider distances between optical flows and establish that the OFM is isometric.

\begin{appxlem}
For affine articulations, the OFM at any reference image is globally isometric.
\end{appxlem}

\begin{proof} 
Given two articulations $\theta_1 = (A_1, \bft_1)$ and $\theta_2 = (A_2, \bft_2)$, the Euclidean distance between them is given by
\begin{eqnarray*}
\left(d(f_{\theta_1}, f_{\theta_2})\right)^2 &=& { \int_{\bfx \in {\cal X}} \| f_{\theta_1}(\bfx) - f_{\theta_2} (\bfx) \|_2^2 \, d\bfx } \\
&=& \int_{\bfx \in {\cal X}} \| (A_1 - A_2)\bfx + (\bft_1-\bft_2) \|_2^2  \, d\bfx \\
&=&  (\theta_1 - \theta_2)^T \Sigma (\theta_1 - \theta_2),
\end{eqnarray*}
where
$$\Sigma = \int_{\bfx \in {\cal X}} \left(\left[ \begin{array}{c} \bfx \\ 1 \end{array} \right]
\left[  \bfx^T \;\; 1  \right] \right) d\bfx.$$
For ${\cal X} = [0,\,1]\times[0,\,1]$, it is easily shown that $\Sigma$ is full-rank and positive definite.
Hence, we have
$$d(f_{\theta_1}, f_{\theta_2}) = \| \theta_1 - \theta_2 \|_\Sigma,$$
where $\| \cdot \|_\Sigma$ is the Mahalanobis (or weighted Euclidean) distance defined using the matrix $\Sigma$.
This implies that the OFM is Euclidean and, hence, globally isometric.
\end{proof}

\subsection{Pose manifold}
The pose manifold  is the IAM corresponding to the motion of a camera observing a static scene. It is well-known that the articulation space is 6-dimensional, with 3 degrees of rotation and 3 degrees of translation of the camera, i.e., $\Theta = SO(3) \times \reals^3$.
We assume that the optical flow is the (unique) 2-dimensional projection of the motion flow of the scene induced due to the motion of the camera. 

Without loss of generality, we assume the reference articulation $\theta_\rf = (R_0, \bft_0) = ( {\mathbb I}, {\bf 0})$ and that the camera's internal calibration is known and accounted for \cite{hartley2000multiple}.
At the reference articulation $\theta_\rf$, let the depth at a pixel $\bfx$ be given by $\lambda_\rf(\bfx)$. Under an articulation $\theta = (R_\theta, \bft_\theta)$, the optical flow observed at the pixel $\bfx$ is given by
\begin{equation}
f(\theta) (\bfx)  = P \left( \lambda_\rf(\bfx) R_\theta \left[ \begin{array}{c} \bfx \\ 1 \end{array}\right] + \bft \right) - P \left( \lambda_\rf(\bfx) \left[ \begin{array}{c} \bfx \\ 1 \end{array}\right]  \right)
\label{eqn:poseFlow}
\end{equation}
where $P(\cdot)$ is a projection operator such that $P( x, y, z ) = (x/z, y/z)$. Note that $P$ is a well-defined and {\em infinitely smooth function} provided $z \ne 0$. 

The OFM $O$ is the image of the map $f(\theta) = \{ f(\theta)(\bfx): \; \; \bfx \in {\cal X} \},$ where ${\cal X}$ is the image plane. We now establish the smoothness of the OFM.

\begin{appxlem} \label{lem:lem1}
Consider the OFM corresponding to a pose manifold of a scene with depth map (at the reference articulation $\theta_\rf$) strictly bounded away from zero, i.e, $\forall \: \bfx, \lambda_\rf(\bfx) > \lambda > 0$. Then there exists a neighborhood of $\theta_\rf$ where $f(\theta) = \{ f(\theta)(\bfx); \; \; \bfx \in {\cal X} \}$ is an infinitely smooth map.
\end{appxlem}
\begin{proof}
The smoothness of $f(\theta)$ follows from the smoothness of $f(\theta)(\bfx)$, which in turn follows from the smoothness of the projection operator $P$. Note that $P(x,y,z) = (x/z, y/z)$ is well defined and infinitely smooth only if $z \ne 0$. When all points in the scene have depth bounded away from zero, then there is guaranteed a neighborhood of $\Theta$ around $\theta_\rf$ wherein all points continue to have depth bounded away from zero. This ensures that the projection of all points is well defined, and hence, $f(\theta)$ is infinitely smooth in this neighborhood.
\end{proof}

Note that, in contrast to affine articulations, where the corresponding OFM is smooth globally,
the OFM associated with the pose manifold is smooth only over a neighborhood of the reference point. 

Next, we consider distances between optical flows and establish that the OFM is locally isometric. 
We begin by defining the following representation for rotation matrices.
Let ${\bf \omega} = (\omega_x, \omega_y, \omega_z) \in \reals^3$, and let 
$\Omega_\omega \in \reals^{3 \times 3}$ be the skew-symmetric matrix defined as 
$$\Omega_\omega = \left[
\begin{array}{ccc} 0 & -\omega_x & \omega_y \\ \omega_x & 0 & -\omega_z \\ -\omega_y & \omega_z & 0 \end{array} \right].$$
Noting that the matrix exponential $e^{\Omega_\omega}$ is a rotation matrix, we define our articulations as
$\theta = (\omega, \bft) \in \reals^6$, with $\bft = (t_x, t_y, t_z)$.

Before we state a formal result, we need to introduce two assumptions that will help in  simplifying the derivation.
First, we assume that the imaging model is well-approximated by a weak perspective model \cite{hartley2000multiple}. In the weak perspective model, it is assumed that the variations in the scene depth are significantly smaller than the average depth of the scene. In such a case, the projection map $P(\bfx)=(x/z, y/z)$ can be well-approximated by 
$$P_{\rm wp}(\bfx) = (x/z_{\rm av}, y/z_{\rm av}),$$ where $z_{\rm av}$ is the average scene depth. As a consequence, translations along the $z$-axis are unobservable \cite{hartley2000multiple}; hence, we restrict the articulation space to the rotation of the camera and translation along $x$ and $y$ axes alone.
Second, we assume that the rotation undergone by the camera is small. 
This enables us to approximate the matrix-exponential $e^{\Omega_\omega} \approx {\mathbb I}+\Omega_\omega$. As a consequence of this assumption, we only obtain a local isometry result.
With these two assumptions, we are ready to state a result on the distances between optical flows.

\begin{appxlem}
Consider the pose manifold under the assumption of the weak perspective imaging model.
Then the OFM 
is locally isometric.
\end{appxlem}
\begin{proof}
Without loss of generality, we denote our reference articulation as $\theta_\rf = {\bf 0}$, which corresponds to identity rotation matrix and null translation vector.
We are interested in the distance between two optical flow fields $f(\theta_1)$ and $f(\theta_2)$.
Under our assumptions of weak perspective imaging model and small rotation, we can compute the distance between the optical flows to be
\begin{eqnarray*}
\lefteqn{f(\theta_1)(\bfx) - f(\theta_2)(\bfx)}\\
&=& 
P_{\rm wp} \left( \lambda_\rf R_1 \left[ \begin{array}{c} \bfx \\ 1 \end{array}\right] + \bft_1 \right)  - P_{\rm wp} \left( \lambda_\rf R_2 \left[ \begin{array}{c} \bfx \\ 1 \end{array}\right] + \bft_2 \right) \\
&=&\frac{1}{z_{\rm av}} \left[
\begin{array}{c}
( - \lambda_\rf( \omega_{1,x} y -  \omega_{1, y}) + t_{1, x})-( - \lambda_\rf (\omega_{2,x} y -  \omega_{2, y}) + t_{1, x}) \\
(\lambda_\rf (\omega_{1,x} x - \omega_{1,z}) + t_{1, y}) - (\lambda_\rf(\omega_{2,x} x - \omega_{2,z}) + t_{2, y})
\end{array}
\right] \\
&=& \frac{1}{z_{\rm av}} \left[
\begin{array}{c}
 - \lambda_\rf ( \omega_{1,x}-\omega_{2,x}) y +\lambda_\rf (\omega_{1, y}-\omega_{2,y})  + (t_{1, x} - t_{2,x}) \\
\lambda_\rf (\omega_{1,x} - \omega_{2,x})  x - \lambda_\rf (\omega_{1,z}-\omega_{2, z})  + (t_{1, y} - t_{2, y})
\end{array}
\right]
\end{eqnarray*}
Notice that both $t_{1,z}$ and $t_{2,z}$ are absent in the expression above. As noted above, this is due to the assumption of weak perspective imaging model which makes the translation of the camera along the $z$-axis unobservable.
A key observation is that the articulation parameters $\theta_1 = (\omega_1, \bft_1)$ and $\theta_2 = (\omega_2, \bft_2)$ can be expressed as a function that is linear in the expression above.
Hence, 
$$f(\theta_1)(\bfx) - f(\theta_2)(\bfx) = A(\bfx)\left[ \begin{array}{c} \omega_1-\omega_2 \\ \bft_1 - \bft_2 \end{array} \right].$$
Finally, summing over $\bfx \in {\cal X}$, we obtain
$$d(f(\theta_1), f(\theta_2))^2 = (\theta_1 - \theta_2)^T \left(\int_{\bfx \in {\cal X}} A^T(\bfx)A(\bfx) d\bfx \right) (\theta_1 - \theta_2).$$
Hence, the OFM is locally isometric. 
\end{proof}